\documentclass{article}

\def\shownotes{1}

\usepackage{microtype}
\usepackage{subcaption}
\usepackage{balance}

\usepackage[hidelinks,breaklinks=true]{hyperref}
\usepackage{graphicx}
\usepackage{multirow}
\usepackage{xspace}

\usepackage[font=small,labelfont=bf]{caption}

\usepackage{macro_math}

\newcommand\bertlarge{BERT$_{\small \textsc{LARGE}}$\xspace}

\usepackage[accepted]{icml2020}

\icmltitlerunning{On the Generalization Effects of Linear Transformations in Data Augmentation}

\begin{document}

\twocolumn[
\icmltitle{On the Generalization Effects of Linear Transformations in Data Augmentation}

\icmlsetsymbol{equal}{*}

\begin{icmlauthorlist}
\icmlauthor{Sen Wu}{equal,stanford}
\icmlauthor{Hongyang R. Zhang}{equal,penn}
\icmlauthor{Gregory Valiant}{stanford}
\icmlauthor{Christopher R\'e}{stanford}
\end{icmlauthorlist}

\icmlaffiliation{stanford}{Department of Computer Science, Stanford University}
\icmlaffiliation{penn}{Department of Statistics, University of Pennsylvania.}

\icmlcorrespondingauthor{all authors}{\{senwu, hongyang, gvaliant, chrismre\}@cs.stanford.edu}

\icmlkeywords{Generalization, Data Augmentation}

\vskip 0.3in
]

\printAffiliationsAndNotice{\icmlEqualContribution} %

\begin{abstract}
	Data augmentation is a powerful technique to improve performance in applications such as image and text classification tasks.	Yet, there is little rigorous understanding of why and how various augmentations work.	In this work, we consider a family of linear transformations and study their effects on the ridge estimator in an over-parametrized linear regression setting.
	First, we show that transformations that preserve the labels of the data can improve estimation by enlarging the span of the training data.
	Second, we show that transformations that mix data can improve estimation by playing a regularization effect.
	Finally, we validate our theoretical insights on MNIST.
	Based on the insights, we propose an augmentation scheme that searches over the  space of transformations by how \textit{uncertain} the model is about the transformed data.
	We validate our proposed scheme on image and text datasets.
	For example, our method outperforms random sampling methods by 1.24\% on CIFAR-100 using Wide-ResNet-28-10.
	Furthermore, we achieve comparable accuracy to the SoTA Adversarial AutoAugment on CIFAR-10, CIFAR-100, SVHN, and ImageNet datasets.
\end{abstract}

\section{Introduction}

Data augmentation refers to the technique of enlarging the training dataset through pre-defined transformation functions.
By searching over a (possibly large) space of transformations through reinforcement learning based techniques, augmentation schemes including AutoAugment~\cite{cubuk2018autoaugment} and TANDA~\cite{REHDR17} have shown remarkable gains over various models on image classification tasks.
Recent work has proposed random sampling~\cite{cubuk2019randaugment} and Bayesian optimization~\cite{faster_autoaug} to reduce the search cost, since RL-based techniques are computationally expensive.
Despite the rapid progress of these transformation search methods, precisely understanding their benefits remains a mystery because of a lack of analytic tools.
In this work, we study when and why applying a family of linear transformations helps from a theoretical perspective.
Building on the theory, we develop methods to improve the efficiency of transformation search procedures.

A major challenge to understand the theory behind data augmentation is how to model the large variety of transformations used in practice in an analytically tractable framework.
The folklore wisdom behind data augmentation is that adding more labeled data improves generalization, i.e. the performance of the trained model on test data \cite{SK19}.
Clearly, the generalization effect of an augmentation scheme depends on how it transforms the data.
Previous work has analyzed the effect of Gaussian augmentation \cite{RFCLP19} and feature averaging effect of adding transformed data \cite{DGRSDR19,CDL19}.
However, there is still a large gap between the transformations studied in these works and the ones commonly used in augmentation schemes.

In this work, we consider linear transformations which represent a large family of image transformations.
We consider three categories:
(i) Label-invariant (base) transformations such as rotation and horizontal flip;
(ii) Label-mixing transformations including mixup \cite{ZCDL17,I18}, which produces new data by randomly mixing the features of two data points (e.g. a cat and a dog) -- the labels are also mixed;
(iii) Compositions of (label-invariant) transformations such as random cropping and rotating followed by horizontal flipping.

\begin{figure*}
  \begin{minipage}[b]{0.50\textwidth}
    \centering
    \begin{tabular}{p{2.15cm} c c c}
      \toprule
	  Transformations &  Example			  & Improvement    \\
	  \midrule
	  Label-invariant & Rotate ($F$)  &  $\frac{(\beta^{\top}P_X^{\perp} Fx)^2} {n}$  \\
	  Label-mixing   & Mixup           &  $\frac{\norm{X\beta}^2}{n^2}$ \\
	  Composition     & Rotate and flip $(F_1F_2)$   &  $\frac{(\beta^{\top}P_X^{\perp}F_1 F_2 x)^2} n$\\
	  \bottomrule
	\end{tabular}
    \vspace{0.05cm}
	\captionof{table}{Illustrative examples of our theoretical results. For label-invariant transformations and their compositions, they can reduce the estimation error at a rate proportional to the added information. Label-mixing reduces estimation error through a regularization effect.}
	\label{table_intro}
  \end{minipage}%
  \hfill
  \begin{minipage}[b]{0.46\textwidth}
    \centering
    \includegraphics[width=\textwidth]{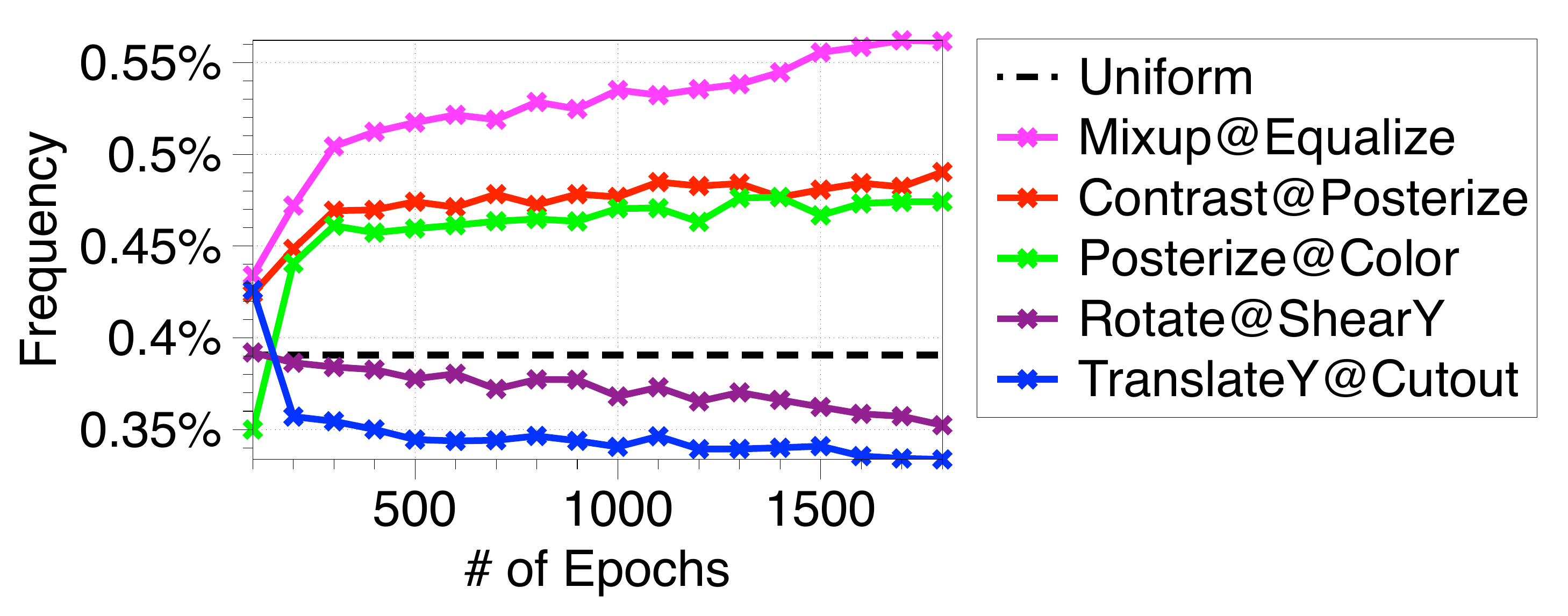}
    \vspace{-0.8cm}
    \captionof{figure}{Our method learns and reduces the frequencies of the better performing transformations during training. Base model: PyramidNet+ShakeDrop~\cite{han2017deep,yamada2018shakedrop}. Dataset: CIFAR-10. See Section \ref{sec_abl} for the details.}
    \label{fig:trans_shift}
  \end{minipage}
\end{figure*}

To gain insight into the effects of linear transformations, we consider a conceptually simple over-parametrized model proposed in \citet{BLLT19,adversarial_aug20,HMRT19} that captures the need to add more data as in image settings.
Suppose we are given $n$ training data points $x_1,\dots,x_n\in\real^p$ as $X \in \real^{n\times p}$ with labels $Y \in \real^n$.
In this setting, the labels obey the linear model under ground truth parameters $\beta \in \real^p$, i.e. $Y = X\beta + \varepsilon$, where $\varepsilon\in \real^n$ denotes i.i.d. random noise with mean zero and variance $\sigma^2$.

Importantly, we assume that $p > n$, hence the span of the training data points does not include the entire space of $\real^p$.
We consider the ridge estimator $\hat{\beta}$ with a fixed $\ell_2$ regularization parameter
and measure the estimation error of the ridge estimator by its distance to $\beta$.

Our first insight is that {\it label-invariant transformations} can {\it add new information} to the training data.
We use our theoretical setup described above to present a precise statement.
For a data point $(x, y)$, a label-invariant transformation $F\in\real^{p\times p}$ in the regression setting produces a new data point $Fx$ with label $y$.
We show that by adding $P_X^{\perp} Fx$ which is outside the span of the training data, we reduce the estimation error of the ridge estimator at a rate proportional to $(\beta^{\top}P_X^{\perp} Fx)^2 / n$ (see Theorem \ref{thm_bias_orth} for the result).
Here
$P_X^{\perp}$ denotes the projection to the orthogonal space of $X^{\top}X$.
In Section \ref{sec_aug}, we validate the insight for classification settings on MNIST by showing that a label-invariant transformation can indeed add new information by reducing the bias (intrinsic error score) of the model.

Our second insight is that {\it label-mixing transformations} can {\it provide a regularization effect}.
In the theoretical setup, given two data points $(x_1, y_1)$ and $(x_2, y_2)$, the mixup transformation with parameter $\alpha$ sampled from the Beta distribution produces $\alpha x_1 + (1-\alpha) x_2$ with label $\alpha y_1 + (1-\alpha) y_2$ \cite{ZCDL17}.
Interestingly, mixup does not {\it add new information} since the mixed data lies in the span of the training data.
However, we show that mixup plays a regularization effect through shrinking the training data relative to the $\ell_2$ regularization.
The final result is that adding the mixup sample reduces estimation error by $\Theta(\norm{X\beta}^2 / n^2)$ (see Theorem \ref{thm_random_mixup} for the result).
In Section \ref{sec_aug}, we validate the insight on MNIST by showing that mixing same-class digits can reduce the variance (instability) of the model.

Finally, for {\it compositions of label-invariant transformations}, we can show their effect for adding new information as a corollary of the base case (see Corollary \ref{cor_compose} for details and we provide the validation on MNIST in Section \ref{sec_aug}).
We provide an illustration of the results in Table \ref{table_intro}.

\vspace{0.075in}
{\bf Algorithmic results.}
	Building on our theory, we propose an uncertainty-based random sampling scheme which, among the transformed data points, picks those with the highest loss, i.e. ``providing the most information''.
	We find that there is a large variance among the performance of different transformations and their compositions.
	Unlike RandAugment~\cite{cubuk2019randaugment}, which averages the effect of all transformations, the idea behind our sampling scheme is to better select the transformations with strong performance.

	We show that our proposed scheme applies to a wide range of datasets and models.
	First, our sampling scheme achieves higher accuracy by finding more useful transformations compared to RandAugment on three different CNN architectures. For example, our method outperforms RandAugment by $0.59\%$ on CIFAR-10 and $1.24\%$ on CIFAR-100 using Wide-ResNet-28-10, and $1.54\%$ on ImageNet using ResNet-50.
	Figure \ref{fig:trans_shift} shows the \textit{key insight} behind our proposed method.
	We compare the frequencies of transformations (cf. Section \ref{sec_exp} for their descriptions) sampled by our method.
	As the training procedure progresses, our method gradually learns transformations providing new information (e.g. Rotate followed by ShearY) and \textit{reduces} their frequencies.
	On the other hand, the frequencies of transformations such as Mixup followed by Equalize  \textit{increase} because they produce samples with large errors that the model cannot learn.

	Second, we achieve similar test accuracy on CIFAR-10 and CIFAR-100 compared to the state-of-the-art Adversarial AutoAugment~\cite{adversarial_autoaugment}.
	By contrast, our scheme is conceptually simpler and computationally more efficient; Since our scheme does not require training an additional adversarial network, the training cost reduces by at least 5x.
	By further enlarging number of augmented samples by 4 times, we achieve test accuracy $85.02\%$ on CIFAR-100, which is higher than Adversarial AutoAugment by $0.49\%$.
	Finally, as an extension, we apply our scheme on a sentiment analysis task and observe improved performance compared to the previous work of \citet{xie2019unsupervised}. %

\textbf{Notations.}
We describe several standard notations that will be used later.
We use the big-O notation $f(n) \le O(g(n))$ to indicate that $f(n) \le C\cdot g(n)$ for a fixed constant $C$ and large enough $n$.
We use $f(n) \lesssim g(n)$ to denote that $f(n) \le O(g(n))$.
For a matrix $X\in\real^{d_1\times d_2}$, let $X^{\dagger}$ denote the Moore-Penrose psuedoinverse of $X$.

\section{Preliminaries}\label{sec_modeling}

Recall that $X = [x_1^{\top}, \dots, x_n^{\top}]\in\real^{n\times p}$ denotes the training data, where $x_i\in\real^{p}$ for $1\le i\le n$.
Let $\id_p\in\real^{p\times p}$ denote the identity matrix.
Let $P_X$ denote the projection matrix onto the row space of $X$. Let $P_X^{\perp} = \id_p - P_X$ denote the projection operator which is orthogonal to $P_X$.
The ridge estimator with parameter $\lambda < 1$ is given by
\begin{align*}
	\hat{\beta}(X, Y) = (X^{\top}X + n\lambda \id)^{-1} X^{\top} Y,
\end{align*}
which arises from solving the mean squared loss
\begin{align*}
  \frac 1 {2n} \min_{w\in \real^p} \normFro{Xw - Y}^2 + \frac {\lambda} 2 \norm{w}^2.
\end{align*}
We use $\hat{\beta}^F$ to denote the ridge estimator when we augment $(X, Y)$ using a transformation function $F$.
The estimation error of $\hat{\beta}$ is given by
$  e(\hat{\beta}) \define \exarg{\varepsilon}{\norm{\hat{\beta} - \beta}^2}.$
Next we define the label-invariance property for regression settings.
\begin{definition}[Label-invariance]
	For a matrix $F\in \real^{p\times p}$, we say that $F$ is label-invariant over $\cX\subseteq\real^{p}$ for $\beta\in\real^p$ if
	\[ x^{\top} \beta = (Fx)^{\top} \beta, \text{ for any } x \in \cX. \]
\end{definition}
\vspace{-0.1in}

As an example, consider a 3-D setting where {$\cX = \set{(a, b, 0) : \forall a, b\in\real}$.
Let $\beta = (1, -1/2, 1/2)^{\top}$ and}
{\small\[ F = \left(\begin{array}{c c c}
	1 & 0 & 0 \\
				0 & \cos\frac{\pi}{2} & \sin\frac{\pi}{2} \\
				0 & -\sin\frac{\pi}{2} & \cos\frac{\pi}{2}
	\end{array}\right).\]}%
\vspace{-0.1in}

Then $(\id - F^{\top})\beta = (0, 0, 1)$ is orthogonal to $\cX$. Hence $F$ is a label-preserving rotation with degree $\frac{\pi}2$ over $\cX$ for $\beta$.

In addition to rotations and mixup which we have described, linear transformations are capable of modeling many image transformations.
We list several examples below.

{\it Horizontal flip.} The horizontal flip of a vector along its center can be written as an orthonormal transformation where
  {\small\begin{align*}
    F = \left(
    \begin{array}{ccccc}
      0 & \dots && 0     & 1 \\
      0 & \dots && 1     & 0 \\
      1 & 0     && \dots & 0
    \end{array}
    \right).
  \end{align*}}
\vspace{-0.2in}

{\it Additive composition}. {For two transformations $F_1$ and $F_2$, their additive composition gives $x^{\aug} = F_1 x + F_2 x$.
	For example, {\it changing the color of an image and adding Gaussian noise} is an additive composition with $F_1 x$ being a color transformation and $F_2 x$ being a Gaussian perturbation.}

{\it Multiplicative composition}. {In this case, $x^{\aug} = F_1 F_2 x$.
	For example, a {\it rotation followed by a cutout} of an image is a multiplicative composition with $F_2$ being a rotation matrix and $F_1$ being a matrix which zeros out certain regions of $x$.}

\section{Analyzing the Effects of Transformation Functions in an Over-parametrized Model}\label{sec_analysis}

How should we think about the effects of applying a transformation?
Suppose we have an estimator $\hat{\beta}$ for a linear model $\beta\in\real^p$.
The bias-variance decomposition of $\hat{\beta}$ is
{\begin{align}
	e(\hat{\beta}) = \underbrace{\bignorm{\exarg{\varepsilon}{\hat{\beta}} - \beta}^2}_{\text{bias}}
	+ \underbrace{\bignorm{\hat{\beta} - \exarg{\varepsilon}{\hat{\beta}}}^2}_{\text{variance}} \label{eq_bv}
\end{align}}
In the context of data augmentation, we show the following two effects from applying a transformation.

\textit{Adding new information.} The bias part measures the error of $\hat{\beta}$ after taking the expectation of $\varepsilon$ in $\hat{\beta}$.
Intuitively, the bias part measures the intrinsic error of the model after taking into account the randomness which is present in $\hat{\beta}$.
A transformation may improve the bias part if $x^{\aug}$ is outside $P_X$.
We formalize the intuition in Section \ref{sec_individual}.

\textit{Regularization.} %
Without adding new information, a transformation may still reduce $e(\hat{\beta})$ by playing a regularization effect.
For example with mixup, $x^{\aug}$ is in $P_X$.
Hence adding $x^{\aug}$ does not add new information to the training data.
However, the mixup sample reweights the training data and the $\ell_2$ regularization term in the ridge estimator.
We quantify the effect in Section \ref{sec_mixup}.

\begin{figure*}[ht!]
	\begin{minipage}[b]{1.00\textwidth}
		\begin{algorithm}[H]
			\small
			\caption{{Uncertainty-based sampling of transformations}}\label{alg_unc}
			\begin{algorithmic}[1]
				\STATE {\bfseries Input.} a batch of $B$ data point $(x_1, y_1)$, $(x_2, y_2)$, \dots, $(x_B, y_B)$

				{\bfseries Require.} $K$ transformations $F_1, F_2,\dots, F_K$, the current model $\hat{\beta}$ and the associated loss $l_{\hat{\beta}}$. Default transformations $G_1,\dots,G_H$.

				{\bfseries Param.} {$L$: number of composition steps; $C$: number of augmented data per input data; $S$: number of selected data points used for training}.

				{\bfseries Return:} a list of $B\cdot S$ transformed data $T$.
				\FOR{$i = 1,\dots, B$}
						\FOR{$k = 1,\dots, C$}
							\STATE Let $F_{j_1}, \dots, F_{j_L}$ be $L$ transformations sampled uniformly at random without replacement from $F_1, \dots, F_K$.
							\STATE Compute $x^{\aug}_k$ and $y^{\aug}_k$ by applying $F_{j_1}, \dots, F_{j_L}$ and the default ones $G_1,\dots, G_H$ sequentially on $(x_i,y_i)$.
							\STATE Infer the loss of $(x^{\aug}_k, y^{\aug}_k)$ as $l^{k} = l_{\hat{\beta}}(x^{\aug}_k, y^{\aug}_k)$.
						\ENDFOR
						\STATE Choose the $S$ data points from $\set{x^{\aug}_k, y^{\aug}_k}_{k=1}^C$ that have the highest losses $l^k$ and add them to $T$.
				\ENDFOR
			\end{algorithmic}
		\end{algorithm}
	\end{minipage}
\end{figure*}

\subsection{Label-Invariant Transformations}\label{sec_individual}

We quantify the effect of label-invariant transformations which add new information to the training data.

{\bf Example.} %
Given a training data point $(x, y)$, let $(x^{\aug}, y^{\aug})$ denote the augmented data where $y^{\aug}=y$.
Intuitively, based on our training data $(X, Y)$, we can infer the label of any data point within $P_X$ (e.g. when $\sigma$ is sufficiently small).
If $x^{\aug}$ satisfies that $P_X x^{\aug} = 0$, then add $(x^{\aug}, y^{\aug})$ does not provide new information.
On the other hand if $P_X^{\perp}x^{\aug} \neq 0$, then adding $x^{\aug}$ expands the subspace over which we can obtain accurate estimation.
Moreover, the added direction corresponds to $P_X^{\perp} x^{\aug}$.
Meanwhile, since $y^{\aug} = y$, it contains a noise part which is correlated with $\varepsilon$.
Hence the variance of $\hat{\beta}$ may increase.

To derive the result, we require a technical twist: instead of adding the transformed data $Fx$ directly, we are going to add its projection onto $P_X^{\perp}$.
This is without loss of generality since we can easily infer the label of $P_X^{\perp} Fx$.
Let $\hat{\beta}^F$ denote the ridge estimator after adding the augmented data to the $(X, Y)$.
Our result is stated as follows.

\begin{theorem}\label{thm_bias_orth}
	Suppose we are given a set of $n$ covariates $X\in\real^{n\times p}$ with labels $Y = X\beta + \varepsilon$, where $\beta\in\real^p$ and $\varepsilon\in\real^n$ has mean $0$ and variance $\sigma^2$.
	Let $F\in\real^{p\times p}$ be a label-invariant transformation over $X$ for $\beta$.
	Let $(x, y)$ be a data point from $(X, Y)$.

	Let $z = P_X^{\perp}Fx$ and $y^{\aug} = y - \diag{(X^{\top})^{\dagger} Fx} Y$.
	Suppose that we augment $(X, Y)$ with $(z, y^{\aug})$.
	Then, we have that
	{\small
	\begin{align}
		e(\hat{\beta}) - e(\hat{\beta}^{F}) \ge&
		\frac {(2\lambda(n+1) - (1-2\lambda)\norm{z}^2)} {\lambda^2 (n + 1 + \norm{z}^2)^2}  \cdot \inner{z}{\beta}^2 \nonumber\\
	&- \frac {(2 + 2 \frac{\norm{P_X Fx}^2} {\mu_{\min}(X)^2})} {\lambda^2(n+1)^2} \cdot \sigma^2\norm{z}^2. \label{eq_thm_new}
	\end{align}}%
	Moreover, when $\norm{z}^2 = o(n)$
	and $\frac {\inner{z}{\beta}^2} {\norm{z}^2} \ge \log{n}(1 + \frac{\norm{P_X Fx}^2}{\mu_{\min}(X)^2})\frac{\sigma^2}{\lambda n}$ (including $\sigma = 0$), we have
	{\small\begin{align}
		       0\le{e(\hat{\beta}) - e(\hat{\beta}^{F}) - (2 + o(1)) \frac {\inner z {\beta}^2} {\lambda n}}\le {\frac{\poly(\gamma/\lambda)}{n^2}}. \label{eq_thm_simple}
	\end{align}}
\end{theorem}
In the above, $\mu_{\min}(X)$ denotes the smallest singular value of $X$.
For a vector $v$, $\diag{v}\in\real^{d\times d}$ denotes a diagonal matrix with the $i$-th diagonal entry being the $i$-th entry of $v$.
$\poly(\gamma/\lambda)$ denotes a polynomial of $\gamma/\lambda$.

Theorem \ref{thm_bias_orth} shows that the reduction of estimation error scales with $\inner{z}{\beta}^2$, the correlation between the new signal and the true model.
Intuitively, as long as $n$ is large enough, then equation \eqref{eq_thm_simple} will hold.
The proof is by carefully comparing the bias and variance of $\hat{\beta}$ and $\hat{\beta}^{\aug}$ after adding the augmented data point.
On the other hand, we remark that adding $Fx$ directly into $X$ does not always reduce $e(\hat{\beta})$, even when $\inner{z}{\beta}^2 = \Theta(1)$ (cf. \cite{adversarial_aug20}).

Another remark is that for augmenting a sequence of data points, one can repeated apply Theorem \ref{thm_bias_orth} to get the result. We leave the details to Appendix \ref{sec_proof_bias}.

{\bf Connection to augmentation sampling schemes.}
We derive a corollary for the idea of random sampling used in RandAugment.
Let $\set{F_i}_{i=1}^K$ be a set of $K$ label-invariant transformations.
We consider the effect of randomly sampling a transformation from $\set{F_i}_{i=1}^K$.

\begin{corollary}\label{cor_uniform}
	In the setting of Theorem \ref{thm_bias_orth}, let $\set{F_i}_{i=1}^K$ be $K$ label-invariant transformations.
	For a data point $(x, y)$ from $(X, Y)$ and $i=1,\dots, K$, let $z_i = P_X^{\perp} F_ix$ and $y_i^{\aug} = y - \diag{(X^{\top})^{\dagger}F_i x} Y$.

	Suppose that $(z, y^{\aug})$ is chosen uniformly at random from $\set{z_i, y_i^{\aug}}_{i=1}^K$.
	Then we have that
	{\small\begin{align*}
		       & \exarg{z, y^{\aug}}{e(\hat{\beta}) - e(\hat{\beta}^{\unif\set{F_i}_{i=1}^K})} \\
		=& \frac {2 + o(1)} K \bigbrace{\sum_{i=1}^K {\frac{\inner{z_i}{\beta}^2}{\lambda n}}}
		+ {\frac{\poly(\gamma/\lambda)}{n^2}}.
	\end{align*}}
\end{corollary}
The proof follows directly from Theorem \ref{thm_bias_orth}.
Corollary \ref{cor_uniform} implies that the effect of random sampling is simply an average over all the transformations.
However, if there is a large variance among the effects of the $K$ transformations, random sampling could be sub-optimal.

{\bf Our proposed scheme: uncertainty-based sampling.} Our idea is to use the sampled transformations more efficiently via an {\it uncertainty-based} sampling scheme.
For each data point, we randomly sample $C$ (compositions of) transformations.
We pick the ones with the highest losses after applying the transformation.
This is consistent with the intuition of Theorem \ref{thm_bias_orth}.
The larger $\inner{z}{\beta}^2$, the higher the loss of $(x^{\aug}, y^{\aug})$ would be under $\hat{\beta}$.

Algorithm \ref{alg_unc} describes our procedure in detail.
In Line 4-6, we compute the losses of $C$ augmented data points.
In Line 8, we select the $S$ data points with the highest losses for training.
For each batch, the algorithm returns $S \dot B$ augmented samples.

\subsection{Label-Mixing Transformations: Mixup}\label{sec_mixup}

We show that mixup plays a regularization effect through reweighting the training data and the $\ell_2$ regularization term.

Specifically, we analyze the following procedure.
Let $\alpha \in [0, 1]$ be sampled from a Beta distribution with fixed parameters \cite{ZCDL17}.
Let $(x_i, y_i)$ and $(x_j, y_j)$ be two data points selected uniformly at random from the training data.
We add the mixup sample $(x^{\aug}, y^{\aug})$ into $(X, Y)$, with $x^{\aug} = \alpha x_i + (1-\alpha) x_j$ and $y^{\aug} = \alpha y_i + (1-\alpha) y_j$.
Additionally, let $(X^{\aug}, Y^{\aug})$ denote the augmented feature matrix and label vector.

We illustrate that adding the mixup sample is akin to shrinking the training data relative to the regularization term.
Assuming that $\sum_{i=1}^n x_i = 0$, we have %
\begin{align}
		& \exarg{x^{\aug}}{{x^{\aug}}{x^{\aug}}^{\top}} = \frac {(1-2\alpha)^2} n  X^{\top}X \nonumber \\
		\Rightarrow& \exarg{x^{\aug}}{{X^{\aug}}^{\top} X^{\aug}} %
							 = \Big(1 + \frac{(1-2\alpha)^2}{n}\Big)X^{\top}X, \label{eq_mixup_lambda_adj1} %
\end{align}
Hence, adding the mixup sample increases the scale of the $X^{\top}X$ term in the covariance matrix. %

\emph{Additionally, to ensure this result holds, we need to adjust the regularization parameter $\lambda$ to $\frac{n}{n + 1} \lambda$ after inserting the mixup sample}.
More precisely, we consider
\begin{align}
	\hat\beta^{\mixup} = \Big({X^{\aug}}^{\top} X^{\aug} + n\lambda\id \Big)^{-1} {X^{\aug}}^{\top} Y^{\aug}.
\end{align}
Compared with OLS, there is now an additional term in the covariance above, which comes from inserting a mixup sample.
Thus, due to this addition, the bias of $\hat\beta^{\mixup}$ will reduce (although the variance may still increase).
Based on this intuition, we describe our result formally below.

\begin{theorem}\label{thm_random_mixup}
	Let $\gamma > 1$ be a fixed constant which does not grow with $n$.
  Let $\set{x_k}_{k=1}^n$ be $n$ training samples which satisfy that $\sum_{k=1}^n{x_k} = 0$ and $\norm{x_k} \le \gamma$, for all $1\le k \le n$.
	Suppose $\bignormFro{X} \neq 0$.
	For any noise variance level $\sigma$ that satisfies
	\[ \sigma^2 \le \frac{n \lambda^4 \bignorm{X\beta}^2 }{2 (\gamma+\lambda)^3 \bignormFro{X}^2} {}, \]
  in expectation over the randomness of $\alpha, x_i, x_j, \varepsilon$, we have that for large enough values of $n$, the following holds:
	\begin{align} \exarg{\alpha, x_i, x_j, \varepsilon}{e(\hat{\beta}) - e(\hat{\beta}^{\mixup})} \ge \frac{\lambda^2 c (\gamma + \lambda)^{-3} \norm{X\beta}^2}{2 n^2} {}{}, \label{eq_mixup_dec}
	\end{align}
	where $c$ is a fixed constant that depends on the parameters of the Beta distribution for generating $\alpha$.
\end{theorem}

We remark that the assumption that $\sum_{i=1}^n{x_i} = 0$ is indeed satisfied in typical image classification settings. This is because a normalization step, which normalizes the mean of every RGB channel to be zero, is usually applied on all the images in practice.

The intuition behind the proof of Theorem \ref{thm_random_mixup} is that the mixup sample shrinks the training data, which reduces the bias of the estimator.
Thus, for small enough $\sigma$, we can ensure that the bias reduction dominates the variance increase.
For proof details, see Appendix \ref{append_mixup}.

\subsection{Compositions of Label-Invariant Transformations}\label{sec_compose}

Our theory can also be applied to quantify the amount of new information added by compositions of transformations. %
We first describe an example to show that taking compositions expands the search space of transformation functions.

{\bf Example.}
	Consider two transformations $F_1$ and $F_2$, e.g. a rotation and a horizontal flip.
	Suppose we are interested in finding a transformed sample $(x^{\aug}, y^{\aug})$ such that adding the sample reduces the estimation error of $\hat{\beta}^{\aug}$ the most.
	With additive compositions, the search space for $x^{\aug}$ becomes
	\[ \set{Fx : x \in \cX, F \in \set{F_1, F_2, F_1 + F_2}}, \]
	which is strictly a superset compared to using $F_1$ and $F_2$.

Based on Theorem \ref{thm_bias_orth}, we can derive a simple corollary which quantifies the incremental benefit of additively composing a new transformation.

\begin{corollary}\label{cor_compose}
	In the setting of Theorem \ref{thm_bias_orth}, let $F_1, F_2$ be two label-invariant transformations.
	For a data point $(x, y)$ from $(X, Y)$ and $i\in\set{1,2}$, let $z_i = P_X^{\perp}F_i x$ and $y_i^{\aug} = y - \diag{(X^{\top})^{\dagger}F_i x} Y$.
	The benefit of composing $F_2$ with $F_1$ is given as follows
	{\small\begin{align*}
		e(\hat{\beta}^{F_1}) - e(\hat{\beta}^{F_1 + F_2})
		=~& (2 + o(1))\frac{\inner{z_1}{\beta}^2 - \inner{z_1 + z_2}{\beta}^2}{\lambda n} \\
		&+ {\frac {\poly(\gamma/\lambda)} {n^2}}.
	\end{align*}}
\end{corollary}
Corollary \ref{cor_compose} implies that the effect of composing $F_2$ with $F_1$ may either be better or worse than applying $F_1$.
This is consistent with our experimental observation (described in Section \ref{sec_aug}).
We defer the proof of Corollary \ref{cor_compose} to Appendix \ref{app_proof_compose}.
We remark that the example and the corollary also apply to multiplicative compositions.

\section{Measuring the Effects of Transformation Functions}\label{sec_aug}

We validate the theoretical insights from Section \ref{sec_analysis} on MNIST \cite{lecun1998gradient}.
To extend our results from the regression setting to the classification setting, we propose two metrics that correspond to the bias and the variance of a linear model.
Our idea is to decompose the average prediction accuracy (over all test samples) into two parts similar to equation \eqref{eq_bv}, including an {\it intrinsic error score} which is deterministic and an {\it instability score} which varies because of randomness.

We show three claims:
i) Label-invariant transformations such as rotations can add new information by reducing the intrinsic error score.
ii) As we increase the fraction of same-class mixup digits, the instability score decreases.
iii) Composing multiple transformations can either increase the accuracy (and intrinsic error score) or decrease it.
Further, we show how to select a core set of transformations.

\begin{figure*}[!t]
	\begin{minipage}[b]{0.38\textwidth}
	  \centering
	  \begin{tabular}{l c c c c}
		\toprule
        & Avg. 	& Error	& Instab.	\\
        & Acc.	& Score	& Score	\\
		\midrule
		Baseline    & 98.08\%   & 1.52\%    & 0.95\% \\
		Cutout		  &	98.31\%   & {1.43}\%		&	0.86\% \\
		RandCrop		&	98.61\% 	&	\textbf{1.01}\%		&	0.88\% \\
		Rotation		&	\textbf{98.65}\%  &  {1.08}\%		& \textbf{0.77}\% \\
		\bottomrule
	\end{tabular}
	\vspace{0.15in}
	\captionof{table}{Measuring the intrinsic error and instability scores of individual transformations on MNIST. The three transformations all reduce the intrinsic error score compared to the baseline.}\label{tab_decomposition}
	\end{minipage}%
	\hfill
	\begin{minipage}[b]{0.25\textwidth}
	  \includegraphics[width=0.98\textwidth]{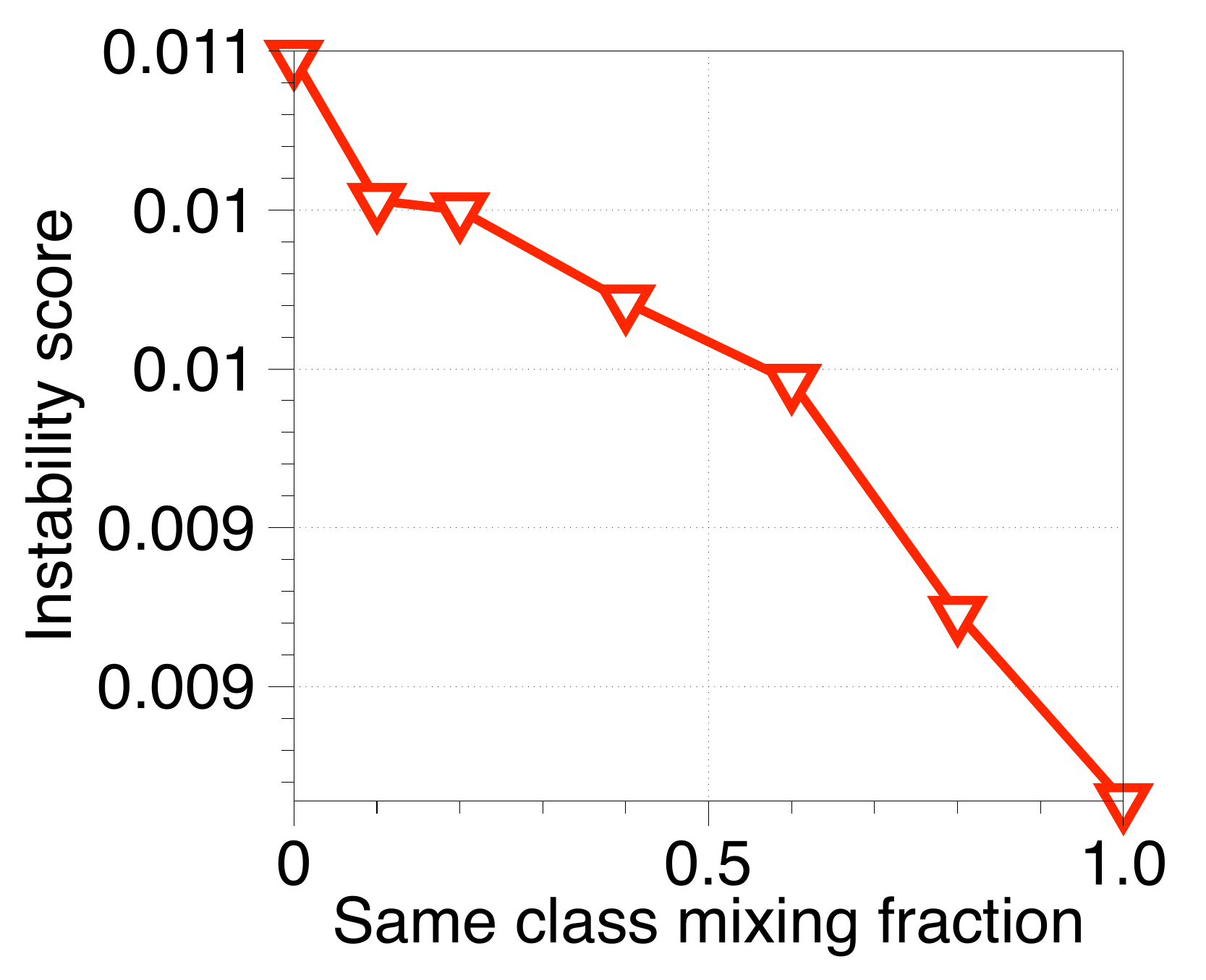}
	  \vspace{-0.0in}
		\caption{The instability score decreases as we increase the fraction of same-class mixup digits on MNIST.}
	  \label{fig_mnist_ins}
	\end{minipage}%
	\hfill
	\begin{minipage}[b]{0.33\textwidth}
	  \centering
	  \includegraphics[width=0.8\textwidth]{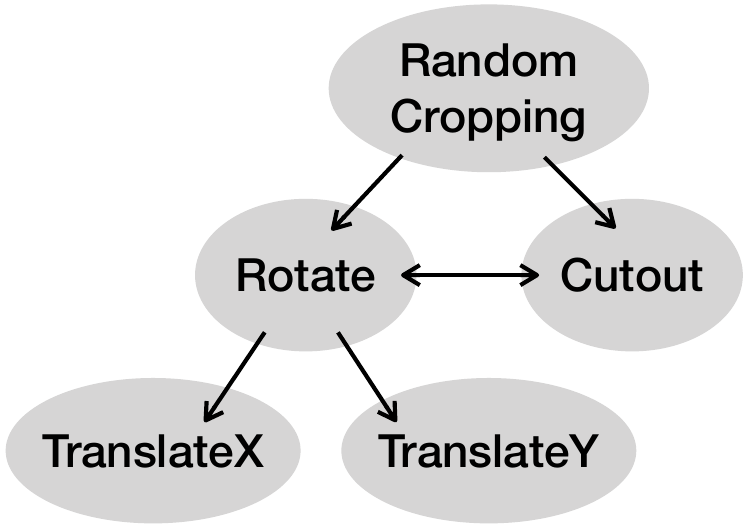}
	\vspace{0.05in}
		\caption{Visualizing which transformations are beneficial for each other. On MNIST, the translations do not provide additional benefit beyond the rest three transformations.}
		\label{fig_compose}
  \end{minipage}
\end{figure*}

{\bf Metrics.} We train $k$ independent samples of multi-layer perceptron with hidden layer dimension $100$.
Let $\hat{\beta}_i$ denote the predictor of the $i$-th sample, for $1\le i\le k$. %
For each data point $x$, let $M(x)$ denote the majority label in $\set{\hat{\beta}_i}_{i=1}^k$ (we break ties randomly).
Clearly, the majority label is the best estimator one could get, given the $k$ independent predictors, without extra information.
Then, we define the {\it intrinsic error score} as
\begin{align}
	\frac 1 n \cdot\bigbrace{ \sum_{(x, y) \in (X, Y)} \ind{M(x) \neq y}}. \label{eq_model_error}
\end{align}
We define the {\it instability score} as
\begin{align}
	\frac 1 {n} \cdot\bigbrace{ \sum_{(x, y) \in (X, Y)} \bigbrace{\frac 1 k \cdot \sum_{i=1}^k \ind{\hat{\beta}_i(x) \neq M(x)}}}. \label{eq_instability}
\end{align}
Compared to equation \eqref{eq_bv}, the first score corresponds to the bias of $\hat{\beta}$, which can change if we add new information to the training data.
The second score corresponds to the variance of $\hat{\beta}$, which measures the stability of the predicted outcomes in the classification setting.
For the experiments we sample $9$ random seeds, i.e. set $k=9$.

{\bf Label-invariant transformations.}
Recall that Section \ref{sec_individual} shows that label-invariant transformations can reduce the bias of a linear model by adding new information to the training data.
Correspondingly, our hypothesis is that the intrinsic error score should decrease after augmenting the training dataset with suitable transformations.

Table \ref{tab_decomposition} shows the result for applying rotation, random cropping and the cutout (i.e. cutting out a piece of the image) respectively.
The baseline corresponds to using no transformations.
We observe that the three transformations all reduce the model error score and the instability score, which confirms our hypothesis.
There are also two second-order effects:
 i) A higher instability score hurts the average predication accuracy of random cropping compared to rotation. ii) A higher intrinsic error score hurts the average prediction accuracy of cutout compared to random cropping.

To further validate that rotated data ``adds new information'', we show that the intrinsic error score decreases by correcting data points that are ``mostly incorrectly'' predicted by the baseline model.
Among the data points that are predicted correctly after applying the rotation but are wrong on the baseline model, the average accuracy of these data points on the baseline model is $25\%$.
In other words, for $75\%$ of the time, the baseline model makes the wrong prediction for these data.
We leave the details to Appendix \ref{app_metric}.

{\bf Label-mixing transformations.}
Section \ref{sec_mixup} shows that in our regression setting, adding a mixup sample has a regularization effect.
What does it imply for classification problems?
We note that our regression setting comprises a single parameter class $\beta$.
For MNIST, there are 10 different classes of digits.
Therefore, a plausible hypothesis following our result is that mixing same-class digits has a regularization effect.

We conduct the following experiment to verify the hypothesis.
We vary the fraction of mixup data from mixing same-class vs. different-class digits.
We expect that as we increase the same class mixing fraction, the instability score decreases.
The results of figure \ref{fig_mnist_ins} confirm our hypothesis.
We also observe that mixing different class images behaves differently compared to mixing same class images.
The details are left to Appendix \ref{app_mixup}.

{\bf Compositions of transformations.} Recall that Section \ref{sec_compose} shows the effect of composing two transformations may be positive or negative.
We describe examples to confirm the claim.
By composing rotation with random cropping, we observe an accuracy of $98.72\%$ with intrinsic error score $0.86\%$, which are both lower than rotation and random cropping individually.
A negative example is that composing translating the X and Y coordinate (acc. $98.0\%$) is worse than translating X (acc. $98.38\%$) or Y (acc. $98.19\%$) individually. Same for the other two scores.

{\it Identifying core transformations.} As an in-depth study, we select five transformation including rotation, random cropping, cutout, translation of the x-axis and the y-axis.
These transformations can all change the geometric position of the digits in the image.
Our goal is to identify a subset of core transformations from the five transformations.
For example, since the rotation transformation implicitly changes the x and the y axis of the image, a natural question is whether composing translations with rotation helps or not.

To visualize the results, we construct a directed graph with the five transformations.
We compose two transformations to measure their model error scores.
An edge from transformation A to transformation B means that applying B after A reduces the model error score by over $10\%$.
Figure \ref{fig_compose} shows the results.
We observe that rotation, random cropping and cutout can all help some other transformations, whereas the translations do not provide an additional benefit for the rest.

\section{Experiments}
\label{sec_exp}

{We test our uncertainty-based transformation sampling scheme on both image and text classification tasks.}
First, our sampling scheme achieves more accurate results by finding more useful transformations compared to RandAugment.
Second, we achieve comparable test accuracy to the SoTA Adversarial AutoAugment on CIFAR-10, CIFAR-100, and ImageNet with less training cost because our method is conceptually simpler.
Finally, we evaluate our scheme in text augmentations to help train a sentiment analysis model.
The code repository for our experiments can be found at https://github.com/SenWu/dauphin.

\subsection{Experimental Setup}\label{sec_exp_setup}

\begin{table*}[h]
\begin{center}
\small
\vspace{0.1cm}
\begin{tabular}{c c | c c c c c c | c}
\toprule
Dataset & Model & Baseline & AA & Fast AA & PBA & RA & Adv. AA & Ours \\
\midrule
\multirow{3}{*}{CIFAR-10}
& Wide-ResNet-28-10      & 96.13 & 97.32 & 97.30 & 97.42 & 97.30 & \textbf{98.10} & 97.89\%($\pm$0.03\%) \\
& Shake-Shake (26 2x96d) & 97.14 & 98.01 & 98.00 & 97.97 & 98.00 & 98.15 & \textbf{98.27\%($\pm$0.05\%)} \\
& PyramidNet+ShakeDrop   & 97.33 & 98.52 & 98.30 & 98.54 & 98.50 & 98.64 & \textbf{98.66\%($\pm$0.02\%)} \\
\midrule
\multirow{1}{*}{CIFAR-100}
& Wide-ResNet-28-10      & 81.20 & 82.91 & 82.70 & 83.27 & 83.30 & 84.51 & \textbf{84.54\%($\pm$0.09\%)} \\
\midrule
SVHN & Wide-ResNet-28-10      & 96.9 & 98.1 & - & - & \textbf{98.3} & - & \textbf{98.3\%($\pm$0.03\%)} \\
\midrule
ImageNet & ResNet-50          & 76.31 & 77.63 & 77.60 & - & 77.60 & \textbf{79.40} & 79.14\% \\
\bottomrule
\end{tabular}
\end{center}
\caption{Test accuracy (\%) on CIFAR-10, CIFAR-100, and SVHN. We compare our method with default data augmentation (Baseline), AutoAugment (AA), Fast AutoAugment (Fast AA), Population Based Augmentation (PBA), RandAugment (RA), and Adversarial AutoAugment (Adv. AA). Our results are averaged over four random seeds except ImageNet experiments.}
\label{exp:image_end2end}
\end{table*}

{\bf Datasets and models.} We consider the following datasets and models in our experiments.

\textit{CIFAR-10 and CIFAR-100}:
The two datasets are colored images with 10 and 100 classes, respectively.
We evaluate our proposed method for classifying images using the following models : Wide-ResNet-28-10~\cite{zagoruyko2016wide}, Shake-Shake (26 2x96d)~\cite{gastaldi2017shake}, and PyramidNet+ShakeDrop~\cite{han2017deep,yamada2018shakedrop}.

{\it Street view house numbers (SVHN)}: This dataset contains color house-number images with 73,257 core images for training and 26,032 digits for testing. We use Wide-ResNet-28-10 model for classifying these images. %

{\it ImageNet Large-Scale Visual Recognition Challenge (ImageNet)}: This dataset includes images of 1000 classes, and has a training set with roughly 1.3M images, and a validation set with 50,000 images. We select ResNet-50~\cite{he2016deep} to evaluate our method.

\vspace{0.042in}
{\bf Comparison methods.} For image classification tasks, we compare Algorithm \ref{alg_unc} with AutoAugment (AA)~\cite{cubuk2018autoaugment}, Fast AutoAugment (Fast AA)~\cite{fast_autoaug}, Population Based Augmentation (PBA)~\cite{ho2019population}, RandAugment (RA)~\cite{cubuk2019randaugment}, and Adversarial AutoAugment (Adv. AA)~\cite{adversarial_autoaugment}.
The baseline model includes the following transformations.
For CIFAR-10 and CIFAR-100, we flip each image horizontally with probability $0.5$ and then randomly crop a $32 \times 32$ sub-image from the padded image.
For SVHN, we apply the cutout to every image. %
For ImageNet, we randomly resize and crop a $224 \times 224$ sub-image from the original image and then flip the image horizontally with probability $0.5$.

\vspace{0.042in}
{\bf Training procedures.} Recall that Algorithm \ref{alg_unc} contains three parameters, how many composition steps ($L$) we take, how many augmented data points ($C$) we generate and how many ($S$) we select for training.
We set $L=2$, $C=4$ and $S=1$ for our experiments on CIFAR datasets and SVHN.
We set $L=2$, $C=8$ and $S=4$ for our experiments on ImageNet.
We consider $K=16$ transformations in Algorithm \ref{alg_unc}, including AutoContrast, Brightness, Color, Contrast, Cutout, Equalize, Invert, Mixup, Posterize, Rotate, Sharpness, ShearX, ShearY, Solarize, TranslateX, TranslateY.
See e.g. \cite{SK19} for descriptions of these transformations.

As is common in previous work, we also include a parameter to set the probability of applying a transformation in Line 4 of Algorithm \ref{alg_unc}.
We set this parameter and the magnitude of each transformation randomly in a suitable range.
We apply the augmentations over the entire training dataset.
We report the results averaged over four random seeds.

\subsection{Experimental Results}

We apply Algorithm \ref{alg_unc} on three image classification tasks (CIFAR10, CIFAR100, SVHN, and ImageNet) over several models.
Table~\ref{exp:image_end2end} summarizes the result. %
We highlight the comparisons to RandAugment and Adversarial AutoAugment since they dominate the other benchmark methods.

{\bf Improving classification accuracy over RandAugment.}
For Wide-ResNet-28-10, we find that our method outperforms RandAugment by 0.59\% on CIFAR-10 and 1.24\% on CIFAR-100.
If we do not use Mixup, our method still outperforms RandAugment by 0.45\% on CIFAR-10.
For Shake-Shake and PyramidNet+ShakeDrop, our method improves the accuracy of RandAugment by 0.27\% and 0.16\% on CIFAR-10, respectively.
For ResNet-50 on ImageNet dataset, our method achieves top-1 accuracy 79.14\% which outperforms RandAugment by 1.54\%.

{\bf Improving training efficiency over Adversarial AutoAugment.}
Our method achieves comparable accuracy to the current state-of-the-art on CIFAR-10, CIFAR-100, and ImageNet.
Algorithm \ref{alg_unc} uses additional inference cost to find the uncertain samples.
And we estimate that the additional cost equals half of the training cost.
{However, the inference cost is 5x cheaper compared to training the adversarial network of Adversarial AutoAugment, which requires generating 8 times more samples for training.}

{\bf Further improvement by increasing the number of augmented samples.}
Recall that Algorithm \ref{alg_unc} contains a parameter $S$ which controls how many new labeled data we generate per training data.
The results in Table \ref{exp:image_end2end} use $C=4$ and $S=1$, but we can further boost the prediction accuracy by increasing $C$ and $S$.
In Table~\ref{exp:image_end2end_2}, we find that by setting $C=8$ and $S=4$, our method improves the accuracy of Adversarial AutoAugment on CIFAR-100 by 0.49\% on Wide-ResNet-28-10. %

\subsection{Ablation Studies}\label{sec_abl}

\begin{figure*}[!t]
	\begin{minipage}[b]{0.45\textwidth}
		\centering
		\begin{tabular}{c | c c}
			\toprule
				Dataset & Adv. AA & Ours ($S=4$) \\
			\midrule
				CIFAR-10  & 98.10($\pm0.15\%$) & \textbf{98.16\%($\pm$0.05\%)} \\
				CIFAR-100 & 84.51($\pm0.18\%$) & \textbf{85.02\%($\pm$0.18\%)} \\
			\bottomrule
		\end{tabular}
		\caption{Increasing the number of augmented data points per training sample can further improve accuracy.}
		\label{exp:image_end2end_2}%
	\end{minipage}%
	\hfill
	\begin{minipage}[b]{0.45\textwidth}
		\centering
		\begin{tabular}{c | c c c}
			\toprule
				&	RA  & Adv. AA &  Ours ($S=1$) \\
			\midrule
				Training ($\times$)  & 1.0 & 8.0 & $\sim1.5$ \\
			\bottomrule
		\end{tabular}
		\caption{Comparing the training cost between our method, RandAugment and Adversarial AutoAugment on CIFAR-10 relative to RandAugment. The training cost of Adversarial AutoAugment is cited from the authors \cite{adversarial_autoaugment}.}
		\label{exp:cost_end2end}
	\end{minipage}
\end{figure*}

{\bf Histgram of selected transformations.}
We examine the transformations selected by Algorithm \ref{alg_unc} to better understand its difference compared to RandAugment.
For this purpose, we measure the frequency of transformations selected by Algorithm \ref{alg_unc} every 100 epoch.
If a transformation generates useful samples, then the model should learn from these samples.
And the loss of these transformed data will decrease as a result.
On the other hand, if a transformation generates bad samples that are difficult to learn, then the loss of these bad samples will remain large.

Figure~\ref{fig:trans_shift} shows the sampling frequency of five compositions.
We test the five compositions on a vanilla Wide-ResNet model.
For transformations whose frequencies are decreasing, we get:
	Rotate and ShearY, 93.41\%;
	TranslateY and Cutout, 93.88\%.
For transformations whose frequencies are increasing, we get:
	Posterize and Color, 89.12\% (the results for other two are similar and we omit the details). %
Hence the results confirm that Algorithm \ref{alg_unc} learns and reduces the frequencies of the better performing transformations.

\subsection{Extension to Text Augmentations}

While we have focused on image augmentations throughout the paper, we can also our ideas to text augmentations.
We extend Algorithm \ref{alg_unc} to a sentiment analysis task as follows.

We choose \bertlarge as the baseline model, which is a 24 layer transformer network from \cite{BERT18}.
We apply our method to three augmentations: back-translation~\cite{yu2018qanet}, switchout~\cite{wang2018switchout}, and word replace~\cite{xie2019unsupervised}.

\vspace{0.07in}
{\it Dataset}.  We use the Internet movie database (IMDb) with 50,000 movie reviews.
The goal is to predict whether the sentiment of the review is positive or negative.

{\it Comparison methods}. We compare with pre-BERT SoTA, \bertlarge, and unsupervised data augmentation (UDA)~\cite{xie2019unsupervised}.
UDA uses \bertlarge initialization and training on 20 supervised examples and DBPedia~\cite{lehmann2015dbpedia} as an unsupervised source.

{\it Results.} We find that our method achieves test accuracy 95.96\%, which outperforms all the other methods by at least 0.28\%.
For reference, the result of using Pre-BERT SoTA is 95.68\%.
The result of using \bertlarge is 95.22\%.
The result of using UDA is 95.22\%.

\section{Related Work}\label{sec_related}

{\bf Image augmentations.}
We describe a brief summary and refer interested readers to the excellent survey by \cite{SK19} for complete references.

Data augmentation has become a standard practice in computer vision such as image classification tasks.
First, individual transformations such as horizontal flip and mixup have shown improvement over strong vanilla models.
Beyond individual transformations, one approach to search for compositions of transformations is to train generative adversarial networks to generate new images as a form of data augmentation \cite{SWL18,LA16,O16,GSZZC18}.
Another approach is to use reinforcement learning based search methods \cite{HTSMX19}.
The work of \citet{cubuk2018autoaugment} searches for the augmentation schemes on a small surrogate dataset.
While this idea reduces the search cost, it was shown that using a small surrogate dataset results in sub-optimal augmentation policies \cite{cubuk2019randaugment}.

The work of \citet{KS18} is closely related to ours since they also experiment with the idea of uncertainty-based sampling.
Their goal is different from our work in that they use this idea to find a representative sub-sample of the training dataset that can still preserve the performance of applying augmentation policies.
Recently, the idea of mixing data has been applied to semi-supervised learning by mixing feature representations as opposed to the input data \cite{BCGPOR19}.

{\bf Theoretical studies}. \citet{DGRSDR19} propose a kernel theory to show that label-invariant augmentations are equivalent to transforming the kernel matrix in a way that incorporates the prior of the transformation.
\citet{CDL19} use group theory to show that incorporating the label-invariant property into an empirical risk minimization framework reduces variance.

Our theoretical setup is related to \citet{adversarial_aug20,RXYDL20}, with several major differences.
First, in our setting, we assume that the label of an augmented data is generated from the training data, which is deterministic.
In their setting, the label of an augmented data includes new information because the random noise is freshly drawn.
Second, we consider the ridge estimator as opposed to the minimum norm estimator, since the ridge estimator includes an $\ell_2$ regularization which is commonly used in practice.
Finally, it would be interesting to extend our theoretical setup beyond linear settings (e.g. \citet{LMZ18,ZSCL19,MRSY19}).

\section{Conclusions and Future Work}

In this work, we studied the theory of data augmentation in a simplified over-parametrized linear setting that captures the need to add more labeled data as in image settings, where there are more parameters than the number of data points. Despite the simplicity of the setting, we have shown three novel insights into three categories of transformations.
We verified our theoretical insights on MNIST. And we proposed an uncertainty-based sampling scheme which outperforms random sampling.
We hope that our work can spur more interest in developing a better understanding of data augmentation methods.
Below, we outline several questions that our theory cannot yet explain.

First, one interesting future direction is to further uncover the mysterious role of mixup (Zhang et al.’17).
Our work has taken the first step by showing the connection between mixup and regularization.
Meanwhile, Table \ref{tab_same_class_mixup} (in Appendix) shows that mixup can reduce the model error score (bias) on CIFAR-10.
Our theory does not explain this phenomenon because our setup implicit assumes a single class (the linear model $\beta$) for all data.
We believe that extending our work to a more sophisticated setting, e.g. mixed linear regression model, is an interesting direction to explain the working of mixup augmentation over multiple classes.

Second, it would be interesting to consider the theoretical benefit of our proposed uncertainty-based sampling algorithm compared to random sampling, which can help tighten the connection between our theory and the proposed algorithm.
We would like to remark that addressing this question likely requires extending the models and tools that we have developed in this work. Specifically, one challenge is how to come up with a data model that will satisfy the label-invariance property for a large family of linear transformations.
We leave these questions for future work.

\section*{Acknowledgements}

We gratefully acknowledge the support of DARPA under Nos. FA86501827865 (SDH) and FA86501827882 (ASED); NIH under No. U54EB020405 (Mobilize), NSF under Nos. CCF1763315 (Beyond Sparsity), CCF1563078 (Volume to Velocity), and 1937301 (RTML); ONR under No. N000141712266 (Unifying Weak Supervision); the Moore Foundation, NXP, Xilinx, LETI-CEA, Intel, IBM, Microsoft, NEC, Toshiba, TSMC, ARM, Hitachi, BASF, Accenture, Ericsson, Qualcomm, Analog Devices, the Okawa Foundation, American Family Insurance, Google Cloud, Swiss Re, the HAI-AWS Cloud Credits for Research program, and members of the Stanford DAWN project: Teradata, Facebook, Google, Ant Financial, NEC, VMWare, and Infosys.
The U.S. Government is authorized to reproduce and distribute reprints for Governmental purposes notwithstanding any copyright notation thereon. Any opinions, findings, and conclusions or recommendations expressed in this material are those of the authors and do not necessarily reflect the views, policies, or endorsements, either expressed or implied, of DARPA, NIH, ONR, or the U.S. Government.

Gregory Valiant's contributions were supported by NSF awards 1804222, 1813049 and 1704417, DOE  award DE-SC0019205 and an ONR Young Investigator Award.

Our experiments are partly run on Stanford University's SOAL cluster hosted in the Department of Management Science and Engineering.

In the previous version of this paper, we should have stated the adjustment of $\lambda$ before/after mixup in Theorem \ref{thm_random_mixup}, which has led to some confusion about this result.
We sincerely thank Kai Zhong (Amazon) for bringing this issue to the author's attention. We have now emphasized this in Section \ref{sec_mixup}.
Thanks also to Hansi Yang (HKUST) for raising several issues in the proof of Theorem \ref{thm_random_mixup}.
These have been fixed in the present paper.

\balance
\bibliographystyle{icml2020}
\bibliography{rf}

\appendix
\onecolumn
\paragraph{Organization of the Appendix.}
In Appendix \ref{app_proofs_sec3}, we fill in the proofs for our main results in Section \ref{sec_analysis}.
We also discuss extensions of our results to augmenting multiple data points and interesting future directions.
In Appendix \ref{app_metric}, we complement Section \ref{sec_aug} with a full list of evaluations on label-invariant transformations.
In Appendix \ref{app_exp}, we describe the implementation and the experimental procedures in more detail.
In Appendix \ref{app_mixup}, we provide further experimental results on why mixup helps improve performance in image classification tasks.

\section{Supplementary Materials for Section \ref{sec_analysis}}\label{app_proofs_sec3}

{\bf Notations.} For a matrix $X\in\real^{d_1\times d_2}$, we use $\mu_{\max}(X)$ to denote its largest singular value.
Let $X_{i, j}$ denote its $(i, j)$-th entry, for $1\le i\le d_1$ and $1\le j \le d_2$.
For a diagonal matrix $D\in\real^{d\times d}$ where $D_{i,i} >0$ for all $1\le i\le d$, we use $\frac 1 D \in\real^{d\times d}$ to denote its inverse.
We use big-O notation $g(n) = O(f(n))$ to denote a function $g(n) \le C \cdot f(n)$ for a fixed constant $C$ as $n$ grows large.
The notation $o(1)$ denotes a function which converges to zero as $n$ goes to infinity.
Let $a\lesssim b$ denote that $a \le C \cdot b$ for a fixed constant $C$.

{\bf The Sherman-Morrison formula.} Suppose $X$ is an invertible square matrix and $u, v\in\real^d$ are column vectors. Then $X + uv^{\top}$ is invertible if and only if $1 + u^{\top}X v$ is not zero. In this case,
\begin{align}
  (X + uv^{\top})^{-1} = X^{-1} - \frac{A^{-1} u v^{\top} A^{-1}}{1 + u^{\top} A^{-1} v}.
\end{align}

\subsection{Proof of Theorem \ref{thm_bias_orth}}\label{sec_proof_bias}

We describe the following lemma, which tracks the change of bias and variance for augmenting $(X, Y)$ with $(z, y)$.
\begin{lemma}\label{lem_bias_change}
	Let $X^{\top}X = UDU^{\top}$ denote the singular vector decomposition of $X^{\top}X$.
	Denote by $\varepsilon^{\aug} = y^{\aug} - z^{\top}\beta$.
	In the setting of Theorem \ref{thm_bias_orth}, we have that
	\begin{align}
		\bias(\hat{\beta})	- \bias(\hat{\beta}_{\aug}) &\ge \frac{(2\lambda(n+1) - (1-2\lambda)\norm{z}^2)\inner{z}{\beta}^2} {\lambda^2(n + 1 + \norm{z}^2)^2} \label{eq_rotate_bias_change} \\
		\var(\hat{\beta})		- \var(\hat{\beta}_{\aug}) &\ge 2\exarg{\varepsilon}{\inner{z\varepsilon^{\aug}}{(\Xaug^{\top}\Xaug + (n+1)\lambda\id)^{-1}X^{\top}\varepsilon}}	\nonumber \\
		&~~~~- \exarg{\varepsilon}{{\varepsilon^{\aug}}^2} \bignorm{(\Xaug^{\top}\Xaug + (n+1)\lambda\id)^{-1} z}^2 \label{eq_rotate_var_change}
	\end{align}
\end{lemma}
\begin{proof}
  The reduction of bias is given by
  \begin{align}
		\bias(\hat{\beta}) - \bias(\hat{\beta}^{\aug})
		=& \bignorm{(X^{\top}X + n\lambda\id)^{-1} X^{\top}X\beta - \beta}^2 - \bignorm{(\Xaug^{\top}\Xaug + (n+1)\lambda\id)^{-1}\Xaug^{\top}\Xaug\beta - \beta} \nonumber\\
		=& \lambda^2 \bigbrace{\bignorm{\bigbrace{\frac{X^{\top}X} n + \lambda\id}^{-1}\beta}^2 - \bignorm{\bigbrace{\frac{\Xaug^{\top}\Xaug}{n+1} + \lambda\id}^{-1}\beta}^2}. \label{eq_bias_red}
  \end{align}
	Denote by $K_{n+1} = (\frac{X^{\top}X}{n+1} + \lambda\id)^{-1}$ and $K_n = (\frac{X^{\top}X}{n} + \lambda\id)^{-1}$
  We observe that
  \begin{align}
		\bigbrace{\frac{\Xaug^{\top}\Xaug}{n+1} + \lambda\id}^{-1}
		&= \bigbrace{\frac{X^{\top}X}{n+1} + \lambda\id + \frac{zz^{\top}}{n+1}}^{-1} \nonumber\\
		&= K_{n+1} -  \frac{\frac 1 {n+1}\cdot K_{n+1} zz^{\top} K_{n+1}} {1 + \frac 1{n+1} z^{\top} K_{n+1} z} \label{eq_bias_expand} \\
		&= K_{n+1} - \frac{\frac 1 {n+1} \frac 1 {\lambda^2} zz^{\top}}{1 + \frac {\norm{z}^2} {\lambda^2(n+1)}} \tag{because $P_X z = 0$} \\
		&= K_n - {U \Delta U^{\top}}  - \frac {\frac 1 {(n+1)\lambda^2} zz^{\top}} {1 + \frac{\norm{z}^2}{n+1}},
	\end{align}
  where we use $\Delta\in\real^{n\times n}$ to denote a diagonal matrix with the $i$-th diagnal entry being 
	\[ \frac 1 {{n(n+1)}}\frac{{D_{i, i}} } {(D_{i, i}/n + \lambda) (D_{i, i}/(n+1) + \lambda)}, \text{ i.e., } \Delta =  (\frac 1 {\frac D n + \lambda} - \frac 1 {\frac D {n+1} + \lambda}). \]
  From this we infer that $\bias(\hat{\beta}) - \bias(\hat{\beta}^{\aug})$ is equal to
  \begin{align*}
		& \lambda^2\bigbrace{2\inner{K_n\beta}{{U\Delta U^{\top}} + \frac{\frac 1{(n+1)\lambda^2} zz^{\top}}{1 + \frac{\norm{z}^2}{n+1}}\beta} -
		\bignorm{\bigbrace{{U\Delta U^{\top}} + \frac{\frac 1 {(n+1)\lambda^2} zz^{\top}}{1 + \frac{\norm{z}^2}{n+1}}}\beta}^2} \\
		=~& \frac{2\inner{z}{\beta}^2} {\lambda (n+1 + \norm{z}^2)} +
		\lambda^2\bigbrace{2\inner{U\frac 1 {D +\lambda} U^{\top}\beta}{{U\Delta U^{\top}} \beta} - \bignorm{{U\Delta U^{\top}\beta}{}}^2} -
		\frac{\norm{z}^2 \inner{z}{\beta}^2}{\lambda^2 (n+1 + \norm{z}^2)^2},
	\end{align*}
  where we again use the fact that $U^{\top}z = 0$.
	We note that for large enough $n$, $\frac 1 {D + \lambda} {\Delta} {}$ strictly dominates ${\Delta^2}$.
  Hence, we conclude that
	\begin{align}
		\bias(\hat{\beta}) - \bias(\hat{\beta}^{\aug})
		\ge& \frac {(2\lambda(n+1) - (1-2\lambda)\norm{z}^2)\inner{z}{\beta}^2} {\lambda^2 (n + 1 + \norm{z}^2)^2}, \label{eq_bias_ge}
	\end{align}
	which proves equation \eqref{eq_rotate_bias_change}.
	For the change of variance,
  We have that
	\begin{align}
		\var(\hat{\beta}) - \var(\hat{\beta}^{\aug})
		=& \exarg{\varepsilon}{\bignorm{(X^{\top}X + n\lambda\id)^{-1}X^{\top}\varepsilon}^2} - \exarg{\varepsilon}{\bignorm{(\Xaug^{\top}\Xaug + (n+1)\lambda\id)^{-1}\Xaug^{\top}\varepsilon^{\aug}}^2} \nonumber \\
		=& \sigma^2 \cdot \bigtr{\bigbrace{(X^{\top}X + n \lambda\id)^{-2} - (\Xaug^{\top}\Xaug + (n+1)\lambda\id)^{-2}} X^{\top}X} \nonumber \\
		& + 2\exarg{\varepsilon}{\inner{z\varepsilon^{\aug}}{(\Xaug^{\top}\Xaug + (n+1)\lambda\id)^{-1}X^{\top}\varepsilon}}
			- \exarg{\varepsilon}{{\varepsilon^{\aug}}^2} \bignorm{(\Xaug^{\top}\Xaug + (n+1)\lambda\id)^{-1} z}^2 \label{eq_var_1}\\
		\ge&  2\exarg{\varepsilon}{\inner{z\varepsilon^{\aug}}{(\Xaug^{\top}\Xaug + (n+1)\lambda\id)^{-1}X^{\top}\varepsilon}}
			- \exarg{\varepsilon}{{\varepsilon^{\aug}}^2} \bignorm{(\Xaug^{\top}\Xaug + (n+1)\lambda\id)^{-1} z}^2. \nonumber
	\end{align}
	Hence equation \eqref{eq_rotate_var_change} is proved by noting that equation \eqref{eq_var_1} is positive.
	This is because $(X^{\top}X + n\lambda\id)^{-2} - (\Xaug^{\top}\Xaug + (n+1)\lambda\id)^{-2}$ is PSD.
	And the trace of the product of two PSD matrices is always positive.
\end{proof}

Based on Lemma \ref{lem_bias_change}, we can prove Theorem \ref{thm_bias_orth}.

\begin{proof}[Proof of Theorem \ref{thm_bias_orth}]
	We note that
	\[ e(\hat{\beta}) - e(\hat{\beta}^{\aug}) = \bias(\hat{\beta}) - \bias(\hat{\beta}^{\aug}) + \var(\hat{\beta}) - \var(\hat{\beta}^{\aug}). \]
	The change of bias has been given by equation \eqref{eq_rotate_bias_change}.
	For equation \eqref{eq_rotate_var_change}, we note that
	$z$ is orthogonal to $(\Xaug^{\top}\Xaug + \lambda\id)^{-1}X^{\top}$.
	Hence the first term is zero.
	The second terms simplifies to
	\begin{align}
		-\frac{\norm{z}^2} {\lambda^2 (n+1)^2} \exarg{\varepsilon}{{\varepsilon^{\aug}}^2}. \label{eq_var_2nd}
	\end{align}
	For $\varepsilon^{\aug}$, denote by $\varepsilon_y = y - x^{\top}\beta$, we note that $\varepsilon^{\aug}$ is equal to
		$\varepsilon_y - \diag{(X^{\top})^{\dagger} Fx}\varepsilon$,
	Hence
	\begin{align*}
		\exarg{\varepsilon}{{\varepsilon^{\aug}}^2} &\le 2\exarg{\varepsilon}{\varepsilon_y^2} + 2\exarg{\varepsilon}{\varepsilon^{\top} \diag{(X^{\top})^{\dagger}(x-z)}^2 \varepsilon} \\
		&= 2\sigma^2 + 2\sigma^2 \diag{(X^{\top})^{\dagger}Fx}^2 \\
		&\le 2\sigma^2(1 + \frac{\norm{P_X Fx}^2} {\mu_{\min}(X)^2}).
	\end{align*}
	Combined with equation \eqref{eq_var_2nd} and equation \eqref{eq_rotate_bias_change}, we have proved equation \eqref{eq_thm_new}.

	For equation \eqref{eq_thm_simple}, by plugging in the assumption that $\norm{z}^2 = o(n)$ and $\frac{\inner{z}{\beta}^2}{\norm{z}^2} \ge (\log n) \cdot (1 + \frac{\norm{P_X Fx}^2}{\mu_{\min}(X)^2}) \sigma^2 / (\lambda n)$ into equation \eqref{eq_thm_new},
	we have that
	\[ e(\hat{\beta}) - e(\hat{\beta}^{\aug}) \ge
	\frac{2\lambda(n+1) - (1-2\lambda)\norm{z}^2 - 10\log n}{\lambda^2(n+1+\norm{z}^2)} \inner{z}{\beta}^2
	\ge 2(1 + o(1)) \frac{\inner{z}{\beta}^2}{\lambda n}. \]

	For the other side of equation \eqref{eq_thm_simple}, from the calculation of equation \eqref{eq_bias_ge}, we have that
	\begin{align*}
		\bias(\hat{\beta}) - \bias(\hat{\beta}^{\aug})
		\le \frac{(2\lambda(n+1) - (1-2\lambda)\norm{z}^2)\inner{z}{\beta}^2}{\lambda^2 (n+1 + \norm{z}^2)} + \bigo{\frac{\poly(\gamma, 1/\lambda)}{n^2}}
	\end{align*}
	For the change of variance, from the calculation of equation \eqref{eq_var_1},
	we have
	\begin{align*}
		\var(\hat{\beta}) - \var(\hat{\beta}^{\aug})
		&\le \sigma^2 \cdot \bigtr{\bigbrace{(X^{\top}X + n\lambda\id)^{-2} - (\Xaug^{\top}\Xaug + (n+1)\lambda\id)^{-2}} X^{\top}X} \\
		&\le \sigma^2 \cdot \bigtr{\bigbrace{(X^{\top}X + n\lambda\id)^{-2} - (X^{\top}X + (n+1)\lambda\id)^{-2}} X^{\top}X} + \bigo{\frac{\sigma^2\cdot\poly(\gamma,1/\lambda)}{n^2}} \\
		&\le \frac{\sigma^2 \cdot \poly(\gamma, 1/\lambda)}{n^2}
	\end{align*}
	Therefore, we have shown that
	\[ e(\hat{\beta}) - e(\hat{\beta}^{\aug})
	\le 2(1 + o(1)) \frac{\inner{z}{\beta}^2}{\lambda n} + \bigo{\frac{\poly(\gamma, 1/\lambda)}{n^2}}. \]
	Hence the proof of the theorem is complete.
\end{proof}

\paragraph{Extension to Adding Multiple Data Points.}
We can also extend Theorem \ref{thm_bias_orth} to the case of adding multiple transformed data points.
The idea is to repeatedly apply Theorem \ref{thm_bias_orth} and equation \eqref{eq_thm_simple}.
We describe the result as a corollary as follows.

\begin{corollary}
	In the setting of Theorem \ref{thm_bias_orth}.
	Let $(x^{\aug}_1, y^{\aug}_1), (x^{\aug}_2, y^{\aug}_2),\dots (x^{\aug}_t, y^{\aug}_t)$ denote a sequence of data points obtained via label-invariant transformations from $(X, Y)$.
	Denote by $(X_0, Y_0) = (X, Y)$.
	For $i = 1,\dots,t$, let $z_i = P^{\perp}_{X_{i-1}} x^{\aug}_i$ and ${y^{\aug}}'_i = y^{\aug}_i - \diag{(X_i^{\top})^{\dagger} x^{\aug}_i} Y_i$.
	Suppose we augment $(X_{i-1}, Y_{i-1})$ with $(z_i, {y^{\aug}}'_i)$.
	Let $(X_i, Y_i)\in (\real^{(n+i)\times p}, \real^{n+i})$ denote the augmented dataset.

	Suppose that $\lambda < 1/4$ is a small fixed constant.
	Let $t = O(n)$ and $\norm{z_i}^2 = o(n)$, for $i = 1,\dots, t$.
	Under the assumption that $\frac{\inner{z_i}{\beta}^2}{\norm{z_i}^2} \gtrsim (1 + \frac{\norm{P_{X_i}x^{\aug}_i}^2}{\mu_{\min}(X_i)^2}) \sigma^2 / (\lambda (n+i))$,
	we have
	\[ e(\hat{\beta}(X, Y)) - e(\hat{\beta}(X_t, Y_t)) \gtrsim \sum_{i=1}^t \frac{\inner{z_i}{\beta}^2}{\lambda n}. \]
\end{corollary}

The proof of the above result follows by repeatedly applying Theorem \ref{thm_bias_orth}.
The details are omitted.

\subsection{Proof of Theorem \ref{thm_random_mixup}}\label{append_mixup}

Similar to the proof of Theorem \ref{thm_bias_orth}, we also track the change of the bias and the variance of $\hat{\beta}$.
Recall that in $\hat{\beta}^{\mixup}$, we adjust the regularization parameter to become $\frac{n}{n + 1} \lambda$; To be clear, we restate the formula for $\hat{\beta}^{\mixup}$ below:
\begin{align}
	\hat\beta^{\mixup} = \Big({X^{\aug}}^{\top}X^{\aug} + n\lambda \id \Big)^{-1} {X^{\aug}}^{\top} Y^{\aug},
\end{align}
where ${X^{\aug}}^{\top}X^{\aug} = X^{\top} X + x^{\aug}{x^{\aug}}^{\top}$, for $x^{\aug} = \alpha x_i + (1 - \alpha) x_j$.
By contrast, the ridge estimator is defined as
\begin{align}
	\hat\beta = (X^{\top} X + n\lambda\id)^{-1} X^{\top} Y.
\end{align}

Compared to our previous result, the difference is that for the mixup transformation, $x^{\aug}$ lies in the row span of $X$.
The following Lemma tracks the change of bias and variance in this setting.

\begin{lemma}\label{lem_mixup}
	In the setting of Theorem \ref{thm_random_mixup}, conditoned on the randomness of $x^{\aug}$, we have that
	{\begin{align}
		\bias(\hat{\beta}) - \bias(\hat{\beta}^{\mixup}) &\ge
		(1 + O(n^{-1})) \lambda^2 \beta^{\top} \Big(\frac{X^{\top}X}n + \lambda \id\Big)^{-2} \frac{x^{\aug} {x^{\aug}}^{\top}} n \Big(\frac{X^{\top} X} n + \lambda \id \Big)^{-1}\beta \label{eq_ridge_bias} \\
		\var(\hat{\beta}) - \var(\hat{\beta}^{\mixup})  &\ge -\sigma^2\cdot \bigtr{(X^{\top}X + n\lambda\id)^{-1} x^{\aug}{x^{\aug}}^{\top} (X^{\top}X + n\lambda\id)^{-1}} \label{eq_ridge_var}
	\end{align}}
\end{lemma}
\begin{proof}[Proof of Lemma \ref{lem_mixup}]
	Let $\varepsilon^{\mixup} = \alpha \varepsilon_i + (1-\alpha) \varepsilon_j$ (recall that $\alpha$ is the mixing proportion between two random samples).
	We first follow the bias-variance decomposition to get that (the details are based on standard regression calculations)
	\begin{align}
		\bias(\hat{\beta}) - \bias(\hat{\beta}^{\mixup})
		= \lambda^2\left(\bignorm{\Big(\frac{X^{\top} X} n + \lambda \id\Big)^{-1}\beta}^2 - \bignorm{\Big(\frac{ {X^{\aug}}^{\top} X^{\aug}} {n} + \lambda \id\Big)^{-1} \beta}^2\right). \label{eq_bias_dec}
	\end{align}
	From equation \eqref{eq_mixup_lambda_adj1}, we have shown that ${X^{\aug}}^{\top} X^{\aug}$ is equal to $1 + \frac{(1 - 2\alpha)^2}{n}$ times $X^{\top} X$ in expectation.
	Thus, equation \eqref{eq_bias_dec} suggests that the bias always decreases after adding one mixup sample.

	Next, denote by
	\begin{align}
		K &= \Big({X^{\top}X} + {n}\lambda\id\Big)^{-1} \label{eq_K}\\
		K^{\aug} &= \Big({X^{\aug}}^{\top}X^{\aug} + n\lambda\id\Big)^{-1}.
	\end{align}
	We can write the change of variance as follows
  \begin{align}
		\var(\hat{\beta}) - \var(\hat{\beta}^{\mixup}) &= \sigma^2 \tr\bigbracket{ {K} X^{\top}X {K}} - \ex{\tr\bigbracket{K^{\aug} {X^{\aug}}^{\top} \varepsilon^{\aug}{\varepsilon^{\aug}}^{\top} {X^{\aug}} K^{\aug}}}  \nonumber \\
		&= \sigma^2\left(\tr\bigbracket{K X^{\top} X K} - \tr\bigbracket{K^{\aug} \Big( (\alpha^2 + (1 - \alpha)^2) x^{\aug} {x^{\aug}}^{\top} + X^{\top} X \Big) K^{\aug}} \right). \label{eq_var_inc}
  \end{align}
	For equation \eqref{eq_var_inc}, recall that $\varepsilon^{\aug} = \alpha \varepsilon_i + (1-\alpha)\varepsilon_j$ and $x^{\aug} = \alpha x_i + (1 - \alpha) x_j$.
	Thus, in expectation over the randomness of $\varepsilon$, we have that
  \begin{align}
    \exarg{\varepsilon}{\Xaug^{\top} \varepsilon^{\aug}{\varepsilon^{\aug}}^{\top} \Xaug}
		&= \sigma^2 \bigbrace{(\alpha^2 + (1-\alpha)^2) x^{\aug} {x^{\aug}}^{\top}  + X^{\top}X } \nonumber\\ %
    &= \sigma^2 \bigbrace{(\alpha^2 + (1-\alpha)^2 ) {x^{\aug}}{x^{\aug}}^{\top} + X^{\top}X}.  \label{eq_mixup_avg} %
  \end{align}
	Now, we can compare equations \eqref{eq_bias_dec} and \eqref{eq_var_inc}.
	At a very high level, we can see that if $\sigma^2$ is small enough (e.g., think of $\sigma$ as zero), then the decreased bias in equation \eqref{eq_bias_dec} must be more than the increased variance in equation \eqref{eq_var_inc}.
	This is the high level idea; In the remainder of the proof, we will work out the details for this.

	From equation \eqref{eq_mixup_avg}, we notice that $K^{\aug}$ is strictly dominated by $K$;
	Thus, we can replace $K^{\aug}$ with $K$ in equation \eqref{eq_var_inc}, and this will make equation \eqref{eq_var_inc} larger as a result (another fact to note is that the trace of the product of two PSD matrices is always positive).
	Thus, we can simplify equation \eqref{eq_var_inc} to the following:
	\begin{align}
		\var(\hat\beta) - \var(\hat\beta^{\mixup})
		\ge& -\sigma^2(\alpha^2 + (1 - \alpha)^2) \tr[K x^{\aug} {x^{\aug}}^{\top} K] \\
		\ge& -\sigma^2 \tr[K x^{\aug} {x^{\aug}}^{\top} K].
	\end{align}
	Recall the definition of $K$ above from equation \eqref{eq_K}.
	This proves equation \eqref{eq_ridge_var}.

	Next, we examine the bias change.
	Let $A = \Big(\frac{X^{\top} X} {n} + \lambda \id\Big)^{-1}$, which is equal to $K$ times $n$.
	By the Sherman–Morrison formula, we have the following %
	\begin{align}
		\Big( \frac{ {X^{\aug}}^{\top} X^{\aug} }{ n} + \lambda\id \Big)^{-1}
		&= A - \frac {A \frac{x^{\aug} {x^{\aug}}^{\top}}{n} A}{1 + \frac 1 n {x^{\aug}}^{\top} A x^{\aug}} \\
		&= A - \Big(1 + O(n^{-1}) \Big) A \frac{x^{\aug} {x^{\aug}}^{\top}} {n} A. \label{eq_Ainv_O}
	\end{align}
	In the last step, to simplify the notation a little bit, we consider $n$ to be some large enough value.
	Plugging equation \eqref{eq_Ainv_O} back into equation \eqref{eq_bias_dec}, we obtain that
	\begin{align}
		\bias(\hat\beta) - \bias(\hat\beta^{\mixup})
		= (1 + O(n^{-1})\lambda^2\beta^{\top} \Big(\frac{X^{\top} X}n + \lambda\id\Big)^{-2} \frac{x^{\aug}{x^{\aug}}^{\top}}n \Big(\frac{X^{\top} X}n + \lambda\id\Big)^{-1} \beta.
	\end{align}
	This completes the proof of equation \eqref{eq_ridge_bias}.
	Taking things together, we thus conclude the proof of this lemma.

\end{proof}

Based on Lemma \ref{lem_mixup}, our next step is to finish the proof of Theorem \ref{thm_random_mixup}.
\begin{proof}[Proof of Theorem \ref{thm_random_mixup}]
	We recall that $x^{\aug} = \alpha x_i + (1-\alpha) x_j$, where $x_i, x_j$ are sampled uniformly at random from the feature vectors in the training set whose size is $n$. Therefore,
  \begin{align}
    \exarg{x_i, x_j}{x^{\aug}{x^{\aug}}^{\top}} &= \frac{\alpha^2 + (1-\alpha)^2}n X^{\top}X + \frac{2\alpha(1 - \alpha)}n \sum_{1\le i \le n} x_i \bigbrace{\sum_{j\neq i} x_j^{\top}} \nonumber \\
    &= \frac{\alpha^2 + (1-\alpha)^2}n X^{\top}X + \frac{2\alpha(1-\alpha)}n \bigbrace{(\sum_{i=1}^n x_i)(\sum_{i=1}^n x_i^{\top}) - X^{\top}X} \nonumber \\
    &= \frac{(1 - 2\alpha)^2}n X^{\top}X, \nonumber
  \end{align}
  where in the last line we used the assumption that $\sum_{i=1}^n {x_i} = 0$.
  For $\alpha$ sampled from a Beta distribution with parameters $a, b$, we have
	\[ \exarg{\alpha}{(1 - 2\alpha)^2} = 1 - \frac{4ab(a+b)}{(a+b)^2(a+b+1)} = \frac{(a-b)^2(a+b)}{(a+b)^2(a+b+1)} + \frac{1}{a+b+1} \]
	which is a constant when $a,b$ are constant values.
	Denote this value by $c$.
  By plugging in the above equation back to equation \eqref{eq_ridge_bias}, we have that
  in expectation over the randomness of $\alpha,x_i,x_j$, the change of bias is at least
  \begin{align*}
		\bias(\hat{\beta}) - \bias(\hat{\beta}^{\mixup})
		&\ge (1 + O(n^{-1})) {} \frac{\lambda^2 c}{n^2} \beta^{\top} \Big(\frac{X^{\top}X}n + \lambda\id\Big)^{-2} X^{\top}X \Big(\frac{X^{\top}X}n + \lambda\id\Big)^{-1}\beta %
	\end{align*}
	Recall that we assume the feature vectors to have Euclidean length at most $\gamma$.
	Thus, the above must be at most
	\begin{align}
		\bias(\hat\beta) - \bias(\hat\beta^{\mixup})
		&\ge (1 + O(n^{-1})) \frac{\lambda^2 c (\gamma + \lambda)^{-3}}{n^2}  \beta^{\top} X^{\top} X \beta \nonumber \\
		&= (1 + O(n^{-1})) \frac{\lambda^2 c (\gamma + \lambda)^{-3}}{n^2} \bignorm{X \beta}^2. \label{eq_bias_conc}
	\end{align}

  Similarly, using equation \eqref{eq_ridge_var}, we obtain that in expectation over the randomness of $z$, the change of variance is at least
  \begin{align}
		\var(\hat{\beta}) - \var(\hat{\beta}^{\mixup})
     &\ge -\frac{\sigma^2 c}{n}\cdot \tr\bigbracket{(X^{\top}X + n\lambda\id)^{-1}X^{\top} X (X^{\top}X + n\lambda\id)^{-1}} \nonumber \\
		&= -\frac{\sigma^2 c}{n^3}\cdot \tr\bigbracket{\Big(\frac{X^{\top}X}n + \lambda\id\Big)^{-1} X^{\top}X \Big(\frac{X^{\top}X}n + \lambda\id\Big)^{-1}} \nonumber \\
		&\ge -\frac{\sigma^2 c \lambda^{-2}}{n^3} \bignormFro{X}^2. \label{eq_var_conc}
  \end{align}
	Now, we compare equation \eqref{eq_var_conc} with equation \eqref{eq_bias_conc};
	We want the former to be smaller than the latter.
	This can be satisfied if $\sigma^2$ is small enough.
	In more detail, as long as
	\begin{align}
		 \sigma^2\le \frac{n \lambda^4 \bignorm{X\beta}^2} {2 (\gamma + \lambda)^3 \bignormFro{X}^2} {}{},%
	\end{align}
	then, combining equations \eqref{eq_var_conc} and \eqref{eq_bias_conc} together, we can conclude that
	\[ e(\hat{\beta}) - e(\hat{\beta}^{\mixup}) \ge \frac{\lambda^2 c (\gamma + \lambda)^{-3} \norm{X\beta}^2  }{2 n^2 }. \]
	This is the same as equation \eqref{eq_mixup_dec}.
	Thus, we have proved that Theorem \ref{thm_random_mixup} holds.

\end{proof}

\paragraph{Remark.} Our proof here crucially depends on the fact that after adding the mixup sample, we are using the adjusted regularization parameter as $\frac{n}{n+1} \lambda$.
Otherwise, the proof would not hold and one can come up with counterexamples for refuting the result.\footnote{Private communication with Kai Zhong (Amazon).}

\paragraph{Discussions.} Our result on the mixup augmentations opens up several immediate open questions.
First, our setting assumes that there is a single linear model $\beta$ for all data.
An interesting direction is to extend our results to a setting with multiple linear models.
This extension can better capture the multi-class prediction problems that arise in image classification tasks.

Second, our result also deals with adding a simple mixup sample.
The limitation is that adding the mixup sample violates the assumption that $\sum_{i=1}^n x_i = 0$ we make for the proof to go through.
However, one would expect that the assumption is still true in expectation after adding the mixup sample.
Formally establishing the result is an interesting question.

\subsection{Proof of Corollary \ref{cor_compose}}\label{app_proof_compose}

\begin{proof}[Proof of Corollary \ref{cor_compose}]
	The proof is by applying equation \eqref{eq_thm_simple} over $z_1 + z_2$ and $z_1$ twice.
	The result follows by adding up results from the two.
	The details are omitted.
\end{proof}

\subsection{The Minimum Norm Estimator}\label{app_min_norm}

In addition to the ridge estimator, the mininum norm estimator has also been studied in the over-parametrized linear regression setting \cite{BLLT19,adversarial_aug20}.
The minimum norm estimator is given by solving the following.
\begin{align*}
  \min_{w\in \real^p} ~& \norm{w}^2 \\
  \text{s.t.}         ~~& Xw = Y.
\end{align*}
It is not hard to show that the solution to the above is given as $\hat{\beta} = (X^{\top}X)^{\dagger} X^{\top} Y$.
We have focused on the ridge estimator since it is typically the case in practice that an $l_2$-regularization is added to the objective.
In this part, we describe the implications of our results for the minimum norm estimator.

For this case, we observe that
\[ \exarg{\varepsilon}{\hat{\beta}} - \beta = (X^{\top}X)^{\dagger} X^{\top} X\beta - X = - P_X^{\perp} \beta, \]
where $P_X^{\perp}$ is the projection orthogonal to $X^{\top}X$.
Hence we can simplify the bias term to be $\beta^{\top} P_X^{\perp} \beta$.
The variance term is $\tr((X^{\top}X)^{\dagger})$.

\textbf{Dealing with psuedoinverse.} We need the following tools for rank-one updates on matrix psuedoinverse.
Suppose $X \in \real^{d\times d}$ is not invertible and $u\in\real^d$ are column vectors. Let $v = P_X u$ denote the projection of $u$ onto the subspace of $X$. Let $s = P_X^{\perp} u$ denote the projection of $u$ orthogonal to $X$.
Suppose that $1 + u^{\top}X^{\dagger}u \neq 0$ and $s, v$ are nonzero. In this case,
\begin{align}
  (X + uu^{\top})^{\dagger} & = X^{\dagger} - X^{\dagger}v {s^{\dagger}}^{\top} - s^{\dagger}v^{\top}X^{\dagger} + (1 + u^{\top}X^{\dagger}u) s^{\dagger}{s^{\dagger}}^{\top} \nonumber \\
  & = X^{\dagger} - X^{\dagger} u {s^{\dagger}}^{\top} - s^{\dagger} u^{\top}X^{\dagger} + (1 + u^{\top}X^{\dagger} u) s^{\dagger}{s^{\dagger}}^{\top}. \label{eq_pinv_rank1}
\end{align}
If $s$ is zero, in other words $u$ is in the row span of $X$, then
\begin{align*}
  (X + uu^{\top})^{\dagger} = X^{\dagger}  - \frac{X^{\dagger} uu^{\top} X^{\dagger}}{1 + u^{\top}X^{\dagger}u}.
\end{align*}
The proof of these identities can be found in \cite{riedel92}.

\textbf{Label-invariant transformations.} When $\sigma = 0$, i.e. there is no label in the labels, we can get a result qualitatively similar to Theorem \ref{thm_bias_orth} for the minimum norm estimator.
This is easy to see because the bias term is simply $\beta^{\top} P_X^{\perp} \beta$.
When we add $z$ which is orthogonal to the span of $X$, then we get that
\[ P_{\Xaug}^{\perp} = P_X^{\perp} - \frac{zz^{\top}}{\norm{z}^2}. \]
Hence the bias of $\hat{\beta}$ reduces by $\frac{\inner{z}{\beta}^2}{\norm{z}^2}$.

When $\sigma \neq 0$, the variance of $\hat{\beta}$ can increase at a rate which is proportional to $1 / \norm{z}^2$.
This arises when we apply \eqref{eq_pinv_rank1} on $(X^{\top}X + zz^{\top})^{\dagger}$.
In Theorem \ref{thm_bias_orth}, we can circumvent the issue because the ridge estimator ensures the minimum singular value for $P_X^{\perp}$ is at least $\lambda$.

\textbf{The mixup transformations.} For the minimum norm estimator, we show that adding a mixup sample always increases the estimation error for $\beta$.
This is described as follows.

\begin{proposition}
  Let $X_{\aug}$ and $Y_{\aug}$ denote the augmented data after adding $(x^{\aug}, y^{\aug})$ to $(X, Y)$. Then for the minimum norm estimator, we have that: a) the bias is unchanged; b) the variance increases, for any constant $\alpha \in (0, 1)$ which does not depend on $n$.
\end{proposition}

\begin{proof}
  For Claim a), we note that since $P_X^{\perp} x_{\aug} = 0$, we have $P_{X_{\aug}}^{\perp} = P_{X}^{\perp}$. Hence, the bias term stays unchanged after adding $(x_{\aug}, y_{\aug})$.

  For Claim b), we note that the variance of $\hat{\beta}$ for the test error is given by
  \begin{align*}
    \tr\bigbracket{(X_{\aug}^{\top}X_{\aug})^{\dagger} (X_{\aug}^{\top}X_{\aug})^{\dagger} X_{\aug}^{\top} \varepsilon^{\aug}{\varepsilon^{\aug}}^{\top} X_{\aug}}
  \end{align*}
  Taking the expectation over $\varepsilon$, we have (c.f. equation \eqref{eq_mixup_avg})
  \begin{align*}
		\exarg{\varepsilon}{\Xaug^{\top}\varepsilon^{\aug}{\varepsilon^{\aug}}^{\top}\Xaug}
		= \sigma^2 \cdot \bigbrace{(\alpha^2 + (1-\alpha)^2 + 2)x^{\aug}{x^{\aug}}^{\top} + X^{\top}X}.
	\end{align*}
  Meanwhile, for $(X_{\aug}^{\top}X_{\aug})^{\dagger}$, we use the Sherman-Morrison formula to decompose the change of the rank-1 update $x_{\aug}x_{\aug}^{\top}$.
  Since $P_X^{\perp}x_{\aug} = 0$ and $x_{\aug}^{\top} K x_{\aug} \lesssim \norm{x_{\aug}}^2 / n < 1$.
  Denote by $K = (X^{\top}X)^{\dagger}$, we have
  \begin{align*}
    (X_{\aug}^{\top}X_{\aug})^{\dagger} = K - \frac{K x_{\aug} x_{\aug}^{\top} K}{1 + x_{\aug}^{\top} K x_{\aug}}
  \end{align*}
  To sum up, we have that
  \begin{align*}
    ~& \ex{\tr\bigbracket{(X_{\aug}^{\top}X_{\aug})^{\dagger} (X_{\aug}^{\top}X_{\aug})^{\dagger} X_{\aug}^{\top} \varepsilon^{\aug}{\varepsilon^{\aug}}^{\top} X_{\aug}}} \\
    =~ & \sigma^2 \tr\bigbracket{K} + (\alpha^2 + (1-\alpha)^2) \tr\bigbracket{K K x_{\aug}x_{\aug}^{\top}} \cdot (1 + O(\frac 1 n))
  \end{align*}
  Hence the variance of the minimum norm estimator increases after adding the mixup sample.
\end{proof}

\section{Supplementary Materials for Section \ref{sec_aug}}\label{app_metric}

We present a full list of results using our proposed metrics to study label-invariant transformations on MNIST.
Then, we provide an in-depth study of the data points corrected by applying the rotation transformation.
Finally, we show that we can get a good approximation of our proposed metrics by using 3 random seeds.
The training details are deferred to Appendix \ref{app_exp}.

\paragraph{Further results on label-invariant transformations.}
For completeness, we also present a full list of results for label-invariant transformations on MNIST in Table \ref{tab_decomposition_full}.
The descriptions of these transformations can be found in the survey of \cite{SK19}.

\begin{table}[h!]
	\centering
	\begin{tabular}{l c c c c}
		\toprule
								& Avg. Acc.	& Error score	& Instability	\\
		\midrule
		Baseline    & 98.08\% ($\pm$0.08\%)   & 1.52\%    & 0.95\% \\
		Invert      & 97.29\% ($\pm$0.18\%)   & 1.85\%    & 1.82\% \\
		VerticalFlip  &97.30\% ($\pm$0.09\%)	  & 1.84\%	  & 1.82\% \\
		HorizontalFlip&97.37\% ($\pm$0.09\%)  & 1.83\%	  & 1.72\% \\
		Blur        & 97.90\% ($\pm$0.12\%)   & 1.62\%    & 1.24\% \\
		Solarize    & 97.96\% ($\pm$0.07\%)   & 1.51\%    & 1.27\% \\
		Brightness  & 98.00\%  ($\pm$0.10\%)	  & 1.69\%    & 1.05\% \\
		Posterize   & 98.04\% ($\pm$0.08\%)   & 1.59\%    & 1.12\% \\
		Contrast    & 98.08\% ($\pm$0.10\%)	  & 1.61\%	  & 0.93\% \\
		Color       & 98.10\%  ($\pm$0.06\%)	  & 1.59\%	  & 0.88\% \\
		Equalize    & 98.12\% ($\pm$0.09\%)   & 1.53\%    & 0.93\% \\
		AutoContrast& 98.17\% ($\pm$0.07\%)   & 1.50\%    & 0.87\% \\
		Sharpness   & 98.17\% ($\pm$0.08\%)   & 1.45\%    & 0.92\% \\
		Smooth      & 98.18\% ($\pm$0.02\%)   & 1.47\%    & 0.99\% \\
		TranslateY  & 98.19\%  ($\pm$0.11\%)   & 1.27\%    & 1.13\% \\
		Cutout		  &	98.31\%  ($\pm$0.09\%)   & 1.43\%		&	0.86\% \\
		TranslateX  & 98.37\%  ($\pm$0.07\%)   & 1.21\%    & 1.00\% \\
		ShearY      & 98.54\%  ($\pm$0.07\%)   & 1.17\%    & 0.85\% \\
		ShearX      & 98.59\%  ($\pm$0.09\%)   & \textbf{1.00}\%    & 0.84\% \\
		RandomCropping&	98.61\%  ($\pm$0.05\%)		&	1.01\%		&	0.88\% \\
		Rotation		&	\textbf{98.65}\%  ($\pm$0.06\%)		& 1.08\%		&	\textbf{0.77}\% \\
		\bottomrule
	\end{tabular}
	\caption{Measuring the intrinsic error and instability scores of label-invariant transformations on MNIST. The results are obtained using 9 random seeds.}\label{tab_decomposition_full}
\end{table}

Next we describe the details regarding the experiment on rotations.
Recall that our goal is to show that adding rotated data points provides new information.
We show the result by looking at the data points that are corrected after adding the rotated data points.
Our hypothesis is that the data points that are corrected should be wrong ``with high confidence'' in the baseline model, as opposed to being wrong ``with marginal confidence''.

Here is our result. Among the 10000 test data points, we observe that there are 57 data points which the baseline model predicts wrong but adding rotated data corrected the prediction.
And there are 18 data points which the baseline model predicts correctly but adding rotated data.
Among the 57 data points, the average accuracy of the baseline model over the 9 random seeds is 19.50\%.
Among the 18 data points, the average accuracy of the augmented model is 29.63\%.
The result confirms that adding rotated data corrects data points whose accuracies are low (at 19.50\%).

\paragraph{Efficient estimation.} We further observe that on MNIST, three random seeds are enough to provide an accurate estimation of the intrinsic error score.
Furthermore, we can tell which augmentations are effective versus those which are not effective from the three random seeds.

\begin{table}
	\centering
	\begin{tabular}{l c c c c c c c c}
		\toprule
							& Baseline	& HorizontalFlip & Contrast 	& Cutout & TranslateY	& ShearX & RandomCropping & Rotation		\\
		\midrule
		9 seeds		&	1.75\%		& 1.83\% 	      & 1.61\%	&	1.43\% &	1.27\%  & 1.00\%	& 1.01\% &	1.08\%\\
		3 seeds		&	1.72\%		& 2.13\% 		    & 1.78\%  &	1.54\% &	1.49\%	& 1.10\%  & 1.11\% & 1.16\%\\
		\bottomrule
	\end{tabular}
	\caption{The result of estimating the intrinsic error score using 3 seeds is within 16\% of the estimation result using 9 seeds.}\label{tab_estimate}
\end{table}

\section{Supplementary Materials for Section \ref{sec_exp}}\label{app_exp}

We describe details on our experiments to complement Section~\ref{sec_exp}.
\begin{itemize}
	\item In Appendix~\ref{app_exp_datasets}, we review the datasets used in our experiments.
	\item In Appendix~\ref{app_exp_models}, we describe the models we use on each dataset.
	\item In Appendix~\ref{app_exp_train}, we summarize the training procedures for all experiments.
\end{itemize}

\subsection{Datasets}
\label{app_exp_datasets}

We describe the image and text datasets used in the experiments in more detail.

\begin{itemize}
	\item \textbf{Image datasets.} For image classification tasks, the goal is to determine the object class shown in the image. We select four popular image classification benchmarks in our experiments.

	\item
\textbf{MNIST.} MNIST~\cite{lecun1998gradient} is a handwritten digit dataset with 60,000 training images and 10,000 testing images. Each image has size $28 \times 28$ and belongs to 1 of 10 digit classes. %

\item
\textbf{CIFAR-10.} CIFAR-10~\cite{krizhevsky2009learning} has 60,000 $32 \times 32$ color images in 10 classes, with 6,000 images per class where the training and test sets have 50,000 and 10,000 images, respectively.

\item
		\textbf{CIFAR-100.} CIFAR-100 is similar to CIFAR-10 while CIFAR-100 has 100 more fine-grained classes containing 600 images each. %

	\item
          \textbf{Street view house numbers (SVHN).} This dataset contains color house-number images with 73,257 core images for training and 26,032 digits for testing which is similar in flavor to MNIST.%

				\item
          \textbf{ImageNet Large-Scale Visual Recognition Challenge (ImageNet).} This dataset is a subset of original ImageNet dataset which has over 15M labeled high-resolution images belonging to roughly 22,000 categories. In this dataset, it includes images of 1000 classes, and has roughly 1.3M training images and 50,000 validation images.

				\item
\textbf{Text dataset.} For text classification task, the goal is to understand the sentiment opinions expressed in the text based on the context provided. We choose \textit{Internet movie database (IMDb)}.

\item
     \textbf{Internet movie database (IMDb).} IMDb~\cite{maas-EtAl} consists of a collection of 50,000 reviews, with no more than 30 reviews per movie. The dataset contains 25,000 labeled reviews for training and 25,000 for test.%
\end{itemize}

\subsection{Models}
\label{app_exp_models}

We describe the models we use in the our experiments.

\textbf{Image datasets.} For the experiments on image classification tasks, we consider five different models including multi-layer perceptron (MLP), Wide-ResNet, Shake-Shake (26 2x96d), PyramidNet+ShakeDrop, and ResNet.
\begin{itemize}

	\item     For the MLP model, we reshape the image and feed it into a two layer perceptron, followed by a classification layer.

	\item     For the Wide-ResNet model, we use the standard Wide-ResNet model proposed by~\cite{zagoruyko2016wide}. We select two Wide-ResNet models in our experiment: Wide-ResNet-28-2 with depth 28 and width 2 as well as Wide-ResNet-28-10 with depth 28 and width 10.

	\item     For Shake-Shake, we use the model proposed by~\cite{gastaldi2017shake} with depth 26, 2 residual branches. The first residual block has a width of 96, namely Shake-Shake (26 2x96d).

	\item     For PyramidNet, we choose the model proposed by~\cite{han2017deep} with ShakeDrop~\cite{yamada2018shakedrop}. We set the depth as 272 and alpha as 200 in our experiments. We refer to the model as PyramidNet+ShakeDrop.

	\item     For ResNet, we select the model proposed by~\cite{he2016deep} with 50 layers. We refer to the model as ResNet-50.

	\item	For CIFAR-10 and CIFAR-100, after applying the transformations selected by Algorithm~\ref{alg_unc}, we apply randomly cropping, horizontal flipping, cutout and mixup.

	\item     For ImageNet, we apply random resized cropping, horizontal flipping, color jitter and mixup after performing Algorithm~\ref{alg_unc}.

	\item     Similarly, for SVHN we apply randomly cropping and cutout after using Algorithm~\ref{alg_unc}.
\end{itemize}

\paragraph{Text dataset.} For the experiments on text classification tasks, we use a state-of-the-art language model call BERT~\cite{BERT18}. We choose the \bertlarge uncased model, which contains 24 layers of transformer.

\subsection{Traning procedures.}
\label{app_exp_train}

We describe the training procedures for our experiments.

\paragraph{Implementation details.} During the training process, for each image in a mini-batch during the optimization, we apply Algorithm~\ref{alg_unc} to find the most uncertain data points and use them to training the model.
We only use the augmented data points to train the model and do not use the original data points.
The specific transformations we use and the default transformations are already mentioned in Section~\ref{sec_exp_setup}.

\paragraph{Image classification tasks.} We summarize the model hyperparameters for all the models in Table~\ref{tab:settings}.

\begin{table*}
\begin{center}

\begin{tabular}{c c c c c c c}
\toprule
Dataset & Model & Batch Size & LR & WD & LR Schedule & Epoch \\
\midrule
MNIST  & MLP & 500 & 0.1 & 0.0001 & cosine & 100 \\
\midrule
CIFAR-10 & Wide-ResNet-28-2       & 128 &  0.1 &  0.0005 & multi-step &  200 \\
CIFAR-10 & Wide-ResNet-28-10      & 128 &  0.1 &  0.0005 &     cosine &  200 \\
CIFAR-10 & Shake-Shake (26 2x96d) & 128 & 0.01 &   0.001 &     cosine & 1800 \\
CIFAR-10 & PyramidNet+ShakeDrop   &  64 & 0.05 & 0.00005 &     cosine & 1800 \\
\midrule
CIFAR-100 & Wide-ResNet-28-10      & 128 &   0.1 &  0.0005 &     cosine &  200 \\
\midrule
SVHN & Wide-ResNet-28-10 & 128 & 0.005 &  0.005 & cosine &  200 \\
\midrule
ImageNet &     ResNet-50 & 160 &   0.1 & 0.0001 & multi-step &  270 \\
\bottomrule
\end{tabular}
\end{center}
\caption{Model hyperparameters on MNIST and CIFAR-10/CIFAR-100. LR means learning rate, WD means weight decay and LR Schedule means learning rate scheduling.
	The parameters we choose are based on the  parameters used for the model previously.}
\label{tab:settings}
\end{table*}

\paragraph{Text classification tasks.} For text classification experiments, we follow the \bertlarge fine-tune procedure proposed by~\cite{BERT18}. We set the dropout rate of 0.1 and apply grid search to tune the learning rate from $\{1e^{-5}, 2e^{-5}, 3e^{-5}\}$, batch size of 32 and 64, and the number of epochs from 2 to 10.

\section{Additional Experiments on Mixup}\label{app_mixup}

We present additional experiments to provide micro insights on the effects of applying the mixup augmentation.
First, we consider the performance of mixing same class images versus mixing different class images.
We show that on CIFAR-10, mixup can reduce the intrinsic error score but mixing same class images does not reduce the intrinsic error score.
Second, we consider the effect of mixup on the size of the margin of the classification model.
We show that mixup corrects data points whih have large margins from the baseline model.

\paragraph{Training procedures.} We apply the MLP model on MNIST and the Wide-ResNet-28-10 model on CIFAR-10.
For the training images, there is a default transformation of horizontal flipping with probability 0.5 and random cropping.
Then we apply the mixup transformation.

\paragraph{Comparing mixup with $\ell_2$-regularization.}
We tune the optimal $\ell_2$ regularization parameter on the baseline model.
Then, we run mixup with the optimal $\ell_2$ parameter.
We also consider what happens if we only mix data points with the same class labels.
Since mixup randomly mixes two images, the mixing ratio for same class images is 10\%.
Table \ref{tab_same_class_mixup} shows the results.

We make several observations.
\begin{itemize}
	\item[(i)] Adding $\ell_2$ regularization reduces the instability score. This is expected since a higher value of $\lambda$ reduces the variance part in equation \eqref{eq_bv}.
	\item[(ii)] On MNIST, mixing same class images improves stability compared to mixup. This is consistent with our results in Section \ref{sec_aug}.
Furthermore, optimally tuning the regularization parameter outperforms mixup.
\item[(iii)] on CIFAR-10, mixing same class images as well as mixup reduces the instability score by more than 10\%.
Furthermore, we observe that mixing different class images also reduces the intrinsic error score.
\end{itemize}

\begin{table}[h!]
	\centering
	\begin{tabular}{l c c c c c c c}
		\toprule
		\multirow{2}{*}{Dataset}
		&\multicolumn{3}{c}{MNIST} & \multicolumn{3}{c}{CIFAR-10} \\ \cmidrule[0.4pt](lr{0.2em}){2-4} \cmidrule[0.4pt](lr{0.2em}){5-7}
		& Avg. Acc.	& Error score	& Instability
		& Avg. Acc.	& Error score	& Instability \\
		\midrule
		Baseline                        & 98.07\%$\pm$0.08\%   & 1.52\%    & 0.95\%  & 87.33\%$\pm$0.31\% & 9.61\% & 7.97\% \\
		Tuning $\ell_2$ reg.	&	\textbf{98.31\%}$\pm$0.05\%	  &	\textbf{1.50\%}    &	0.67\%  & 91.23\%$\pm$0.15\% & 6.55\% & 5.50\% \\
		Mixup-same-class w/ $\ell_2$    & 98.15\%$\pm$0.06\%   & 1.78\%    & \textbf{0.61\%}  & 91.59\%$\pm$0.21\% & 6.55\% & \textbf{4.96\%} \\
		Mixup w/ $\ell_2$ reg.& 98.21\%$\pm$0.05\%  & \textbf{1.58\%}    & 0.70\%  & \textbf{92.44\%}$\pm$0.32\% & \textbf{5.71\%} & \textbf{4.75\%} \\
		\bottomrule
	\end{tabular}
	\caption{Measuring the intrinsic error and instability scores of mixing same class images versus mixup. The optimal $\ell_2$ regularization parameter is $10^{-3}$ for MNIST and $5\times 10^{-4}$ for CIFAR-10.}
	\label{tab_same_class_mixup}
\end{table}

\paragraph{Effects over margin sizes on the cross-entropy loss.}
We further explore why mixup can improve the average prediction accuracy on CIFAR-10.
Intuitively, when mixing two images with different losses, the resulting image will have a loss which is lower than the image with the higher loss.
For the classification setting, instead of measuring the losses, we look at the margin, which is the difference between the largest predicted probability and the second largest predicted probability over the 10 classes.
It is well-known that a larger margin implies better robustness and generalization properties \cite{FHT01,HTF09}.

By comparing the margin distribution between the baseline model versus the one after applying the mixup augmentations, we provide several interesting observations which shed more light on why mixup works.
Figure~\ref{fig:cifar10_margin} shows the results.
First, we observe that among the images for which the baseline model predicts incorrectly, mixup is able to correct those images whose margin is very large!
Second, we observe that among the images for which the baseline model predicts correctly, applying mixup increased the size of the margin for the images whose margin is small on the baseline model.

\begin{figure*}[h!]
	\begin{subfigure}[b]{0.28\textwidth}
		\begin{center}
		\includegraphics[width=0.9\columnwidth]{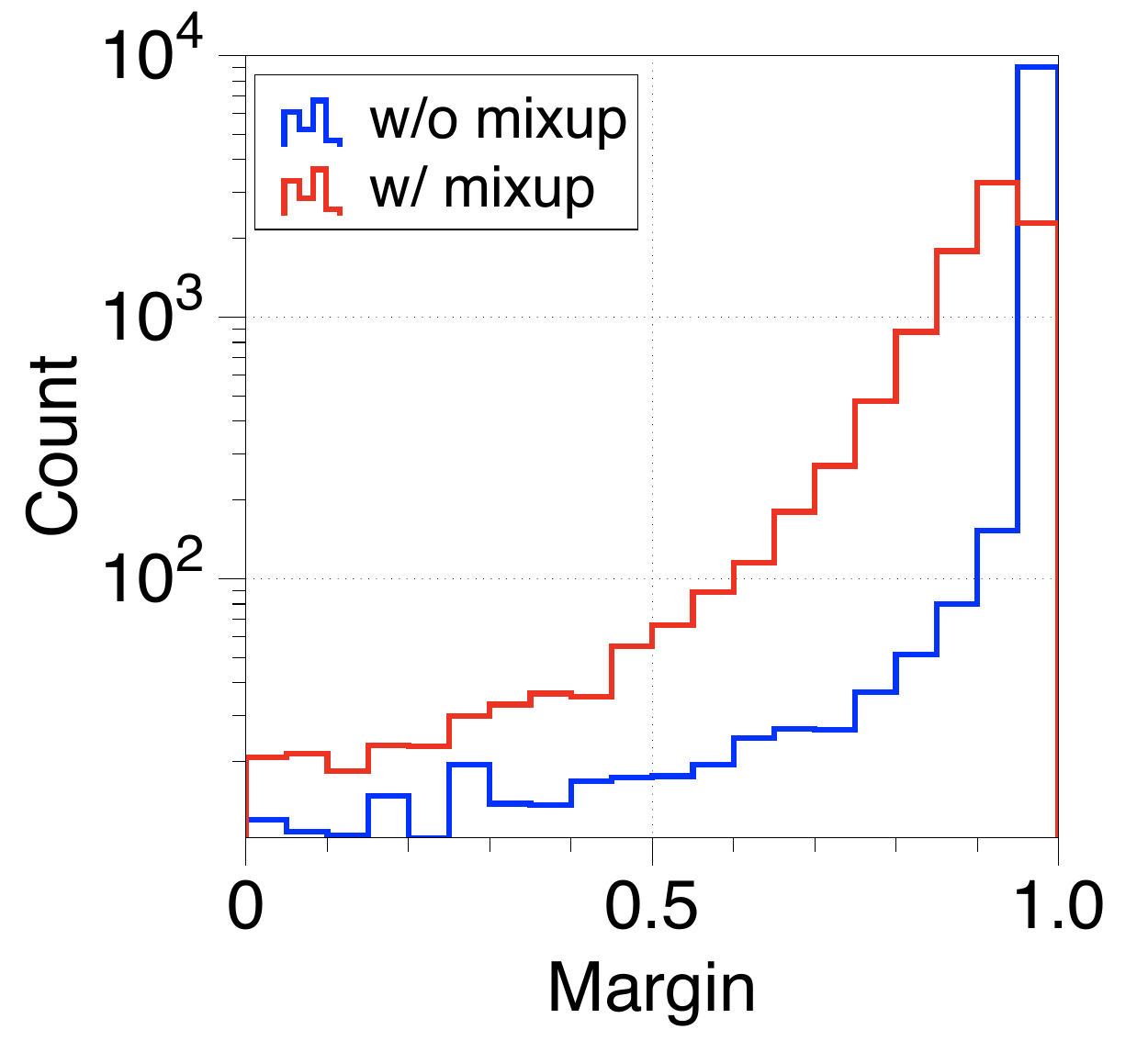}
		\end{center}
		\caption{Correctly predicted images on the baseline model.}
	\end{subfigure}
	\hfill
	\begin{subfigure}[b]{0.28\textwidth}
		\begin{center}
		\includegraphics[width=0.9\columnwidth]{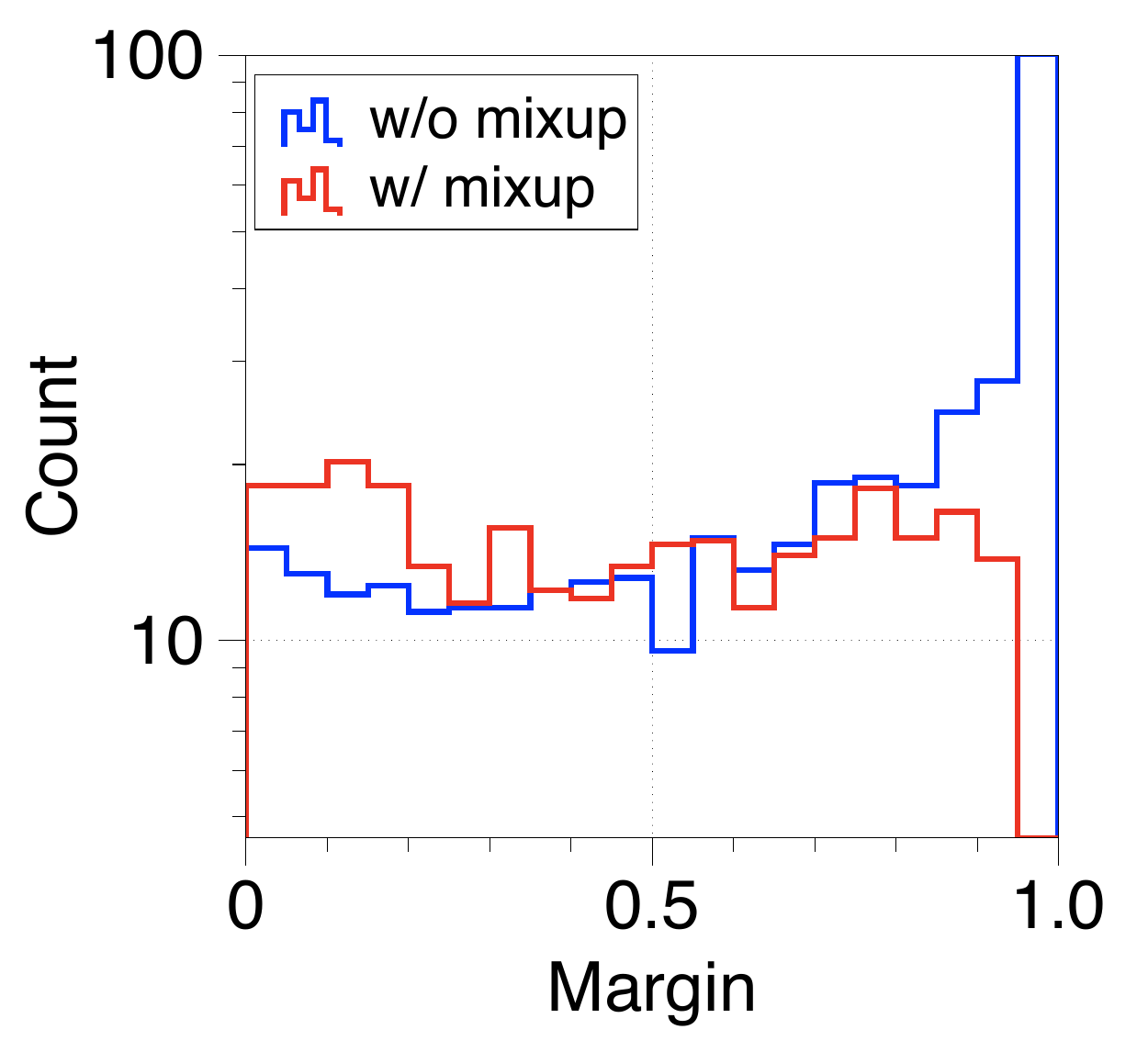}
		\end{center}
		\caption{Incorrectly predicted images on the baseline model.}
	\end{subfigure}
	\hfill
	\begin{subfigure}[b]{0.28\textwidth}
		\begin{center}
		\includegraphics[width=0.9\columnwidth]{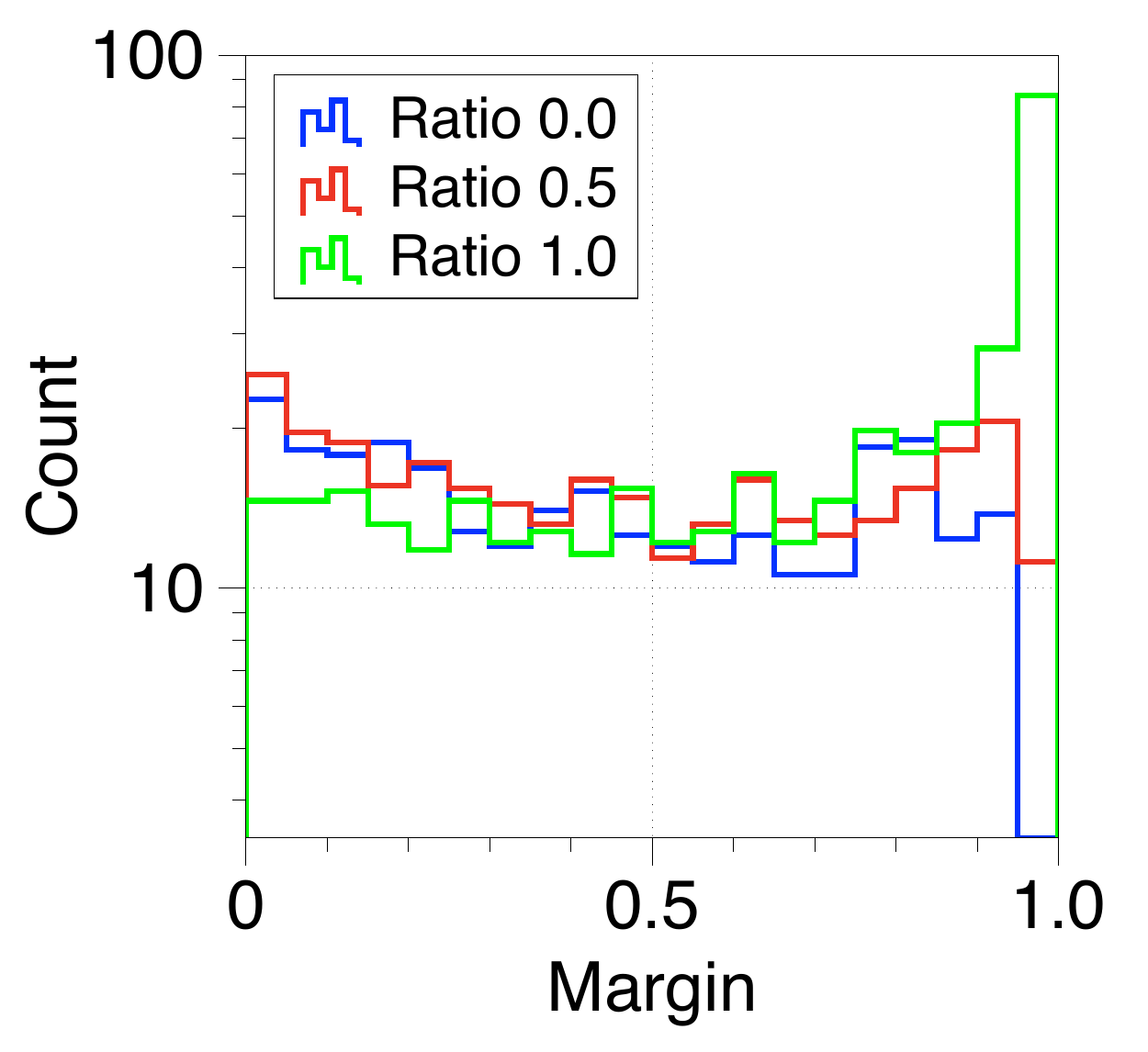}
		\end{center}
		\caption{Corrected images after applying mixup with a certain mixing ratio.}
		\label{fig:cifar10_ratio_margin}
	\end{subfigure}
	\caption{On CIFAR-10, the images corrected by applying the mixup augmentation have large margins on the baseline model. Here the margin is defined as the largest probability minus the second largest probability predicted by the cross entropy loss.
	By varying the mixing ratio, we observe that mixing same class images only does not correct the the data points which are predicted incorrectly with large margins on the baseline model.}
	\label{fig:cifar10_margin}
\end{figure*}

We further observe that the above effect arises from mixing different class images.
We compare the effects of three settings, including i) Mix images from different classes only (Ratio 0.0).
ii) For half of the time, mix images from different classes; Otherwise, mix images from the same classes (Ratio 0.5).
iii) Mix images from the same classes only (Ratio 1.0).
In Figure~\ref{fig:cifar10_ratio_margin}, we observe the effect of correcting large margin data points does not occur when we only mix same class images.

\end{document}